\documentclass{article}
\usepackage{float}
\usepackage{graphicx} 
\usepackage{mdframed}
\usepackage{lipsum}
\usepackage{amsmath,amsthm,amssymb,amsfonts,dsfont}
\usepackage{mathtools}
\usepackage{lineno}

\usepackage{url} 

\usepackage{fullpage}
\usepackage[margin = 2.54cm]{geometry}

\usepackage{microtype}
\usepackage{epsfig}
\usepackage{amssymb}
\usepackage{microtype}
\usepackage{wrapfig}
\usepackage{graphicx}
\usepackage{epsfig,amsmath,amsfonts,amssymb,amsthm,mathrsfs,ifpdf}
\usepackage{indentfirst,relsize}
\usepackage[numbers]{natbib}
\usepackage{setspace}
\usepackage{enumerate}
\usepackage{latexsym}
\usepackage{stackrel}
\usepackage[all]{xy}
\usepackage[usenames,dvipsnames]{pstricks}
\usepackage{pst-grad} 
\usepackage{pst-plot} 
\usepackage{xspace}
\usepackage{bbm}

\usepackage[super]{nth}
\usepackage[pagebackref=true, backref=true]{hyperref}
\usepackage[T1]{fontenc}

\usepackage{cleveref}
\usepackage{enumitem}
\usepackage{float}
\usepackage[figurename=Fig.]{caption}

\usepackage{thmtools} 
\usepackage{thm-restate}

\usepackage[ruled,linesnumbered,vlined,noend]{algorithm2e}
\usepackage{todonotes}

\newcounter{mynotes}
\newcommand{\ignore}[1]{}

\newcommand{\size}[1]{\left| #1 \right|}

\newcommand{\E}{\mathbb{E}}
\newcommand{\remove}[1]{}

\newcommand{\R}{\mathbb{R}}

\newcommand{\cA}{\mathcal{A}}

\newcommand{\cD}{\mathcal{D}}

\newcommand{\Oh}{\mathcal{O}}

\newcommand{\tOh}{\widetilde{{\mathcal O}}}

\newcommand{\pr}{\mathrm{Pr}}

\newcommand{\eps}{\varepsilon}

\newcommand{\Var}{\operatorname{{\bf Var}}}
\newcommand{\Cov}{\operatorname{{\bf Cov}}}

\remove{\usepackage[symbol]{footmisc}

}

\newcommand{\wh}{\widehat}
\newcommand{\wt}{\widetilde}
\newcommand{\paren}[1]{\left(#1\right)}
\newcommand{\braces}[1]{\left\{#1\right\}}
\newcommand{\abs}[1]{\left|#1\right|}

\theoremstyle{plain}
\newtheorem{theo}{Theorem}[section]
\newtheorem{lem}[theo]{Lemma}

\newtheorem{coro}[theo]{Corollary}
\newtheorem{fact}[theo]{Fact}

\newtheorem{cl}[theo]{Claim}
\theoremstyle{definition}
\newtheorem{defi}[theo]{Definition}

\newtheorem{obs}[theo]{Observation}

\newcommand{\kl}{\mathsf{D_{KL}}}
\newcommand{\dfar}{\mathcal{D}_{\mathsf{far}}}
\newcommand{\ham}{\mathsf{Ham}}
\newcommand{\tr}{\mathsf{tr}}

\Crefname{algocf}{Algorithm}{Algorithms}

\usepackage{multicol}

\title{Efficient Sample-optimal Learning of Gaussian Tree Models via Sample-optimal Testing of Gaussian Mutual Information}
\author{Sutanu Gayen\footnote{Indian Institute of Technology, Kanpur. Email: sutanugayen@gmail.com.} \and 
Sanket Kale\footnote{Fujitsu Research. This work was done when the author was a Master's student at the Indian Institute of Technology, Kanpur. Email: kalesanket99@gmail.com.} \and
Sayantan Sen\footnote{Centre for Quantum Technologies, National University of Singapore. Email: sayantan789@gmail.com.}}
\date{}

\begin{document}

\maketitle

\begin{abstract}
    Learning high-dimensional distributions is a significant challenge in machine learning and statistics in this era of big data. Classical research has mostly concentrated on asymptotic analysis of such data when the underlying distribution exhibits structured patterns like product distributions, tree structures, or Bayesian networks. While existing works [Bhattacharyya et al.: SICOMP 2023, Daskalakis et al.: STOC 2021, Choo et al.: ALT 2024] focus on discrete distributions, the current work addresses the tree structure learning problem for Gaussian distributions, providing efficient algorithms with solid theoretical guarantees. This is crucial as real-world distributions are often continuous and differ from the discrete scenarios studied in prior works.

In this work, we first design a conditional mutual information tester for Gaussian random variables that can test whether two Gaussian random variables are independent, or their conditional mutual information is at least $\varepsilon$, for some parameter $\varepsilon \in (0,1)$ using $\mathcal{O}(\varepsilon^{-1})$ samples which we show to be near-optimal. In contrast, an additive estimation would require $\Omega(\varepsilon^{-2})$ samples.
Our upper bound technique uses linear regression on a pair of suitably transformed random variables. Importantly, we show that the chain rule of conditional mutual information continues to hold for the estimated (conditional) mutual information. As an application of such a mutual information tester, we give an efficient $\varepsilon$-approximate structure-learning algorithm for an $n$-variate Gaussian tree model that takes $\widetilde{\Theta}(n\varepsilon^{-1})$ samples which we again show to be near-optimal. In contrast, when the underlying Gaussian model is not known to be tree-structured, we show that $\widetilde{{{\Theta}}}(n^2\varepsilon^{-2})$ samples are necessary and sufficient to output an $\varepsilon$-approximate tree structure.

We also extensively perform experiments for our mutual information testing and structure learning algorithms. Our experiments corroborate our theoretical convergence bounds for the empirical (conditional) mutual information estimators. Finally, we also compare our tree-structured learning result with Graphical Lasso (GLASSO) and Constrained $\ell_1$ Inverse Covariance Estimation (CLIME), two widely used models in the literature, and show that our algorithm outperforms both for learning tree-structured distributions.

\end{abstract}

\newpage

\tableofcontents

\newpage

\section{Introduction}\label{sec:intro}

Learning high-dimensional distributions has been one of the cornerstone problems in the field of machine learning and statistics. In this era of big data, data is being generated in massive amounts, and traditional methods of analyzing data are often not feasible. As a result, over the past two decades, there has been significant interest in developing efficient algorithmic techniques to analyze high dimensional data more efficiently, which also have sound theoretical guarantees. Although learning an arbitrary high dimensional distribution is hard in general, often these tasks become easier when these distributions have some structure. Examples include when these unknown distributions over $n$-dimensional Hamming cube is a \emph{product distribution}, \emph{tree-structured distribution}, or \emph{bounded degree Bayesian networks}.

\emph{Probabilistic graphical models} (see \cite{koller2009probabilistic}) were proposed to address the above situation of succinctly modeling high-dimensional probability distributions. Concrete examples of such models include the class of \emph{Ising models}, Bayesian networks, and Gaussian graphical models. The focus of this work is the Gaussian graphical model. Since its introduction, it has been used to model important datasets across various scientific disciplines such as modeling \emph{brain connectivity networks}: \cite{huang2010learning,varoquaux2010brain}, \emph{gene regulatory networks}: \cite{basso2005reverse,menendez2010gene,schafer2005learning,wille2004sparse} etc. 

A \emph{Gaussian graphical model (GGM)} is simply a multivariate Gaussian distribution $X\sim N(\mu,\Sigma)$, with the mean vector $\mu$ and the covariance matrix $\Sigma$, over the sample space $\mathbb{R}^n$. Its probability density is given by:
\begin{align*}
    \pr[X=x]=\frac{1}{\paren{2\pi\det(\Sigma)}^{\frac{n}{2}}} \exp\paren{
    -\paren{x-\mu}\Sigma^{-1}\paren{x-\mu}^\top
    }
\end{align*}

We can think of the random variable $X$ to be consisting of $n$ (possibly correlated) component random variables $X=(X_1,\dots, X_n)$. 
The inverse of the covariance matrix $\Theta=\Sigma^{-1}$, called the \emph{precision matrix}, encodes the conditional independence relations of the component random variables as follows. First, construct a dependency graph $G_{\Theta}$ on the $n$ vertices: $[n]$, whose $(i,j)$-th entry of the adjacency matrix $A$ is $1$  iff $\Theta_{i,j}\neq 0$. Let $\mathrm{nbr}(i)$ be the set of nodes adjacent to $i$ in $G_{\Theta}$. Then, any particular component $X_i$, conditioned on the set of neighboring components $ X_{\mathrm{nbr}(i)}:=\{X_j:j\in \mathrm{nbr}(i)\}$ is independent of all other components:
\begin{align*}
    X_i \bot X_{\mathrm{nbr}(i)} \mid X\setminus \paren{\{X_i\}\cup  X_{\mathrm{nbr}(i)}}.
\end{align*}

Specifically, in this paper, we consider the set of \emph{tree-structured Gaussian graphical models} where the graph $G_{\Theta}$ is \emph{tree-structured}, namely it does not have any cycles. Perhaps, this is the simplest class of GGMs one can think of apart from the class of product distributions where $\Sigma$ is diagonal and the graph $G_{\Theta}$ does not have any edges. 
Tree-structured graphical models are used in Biology to model \emph{phylogenetic networks}~\citep[Page 53]{jones2021bayesian}, \cite{boix2022chow} and references therein. \cite{chow1968approximating} in a classic work have shown that when the underlying GGM is tree-structured, the underlying tree structure $G_{\Theta}$ coincides with the maximum spanning tree of the complete graph whose weight of the edge $(i,j)$ is given by the mutual information $I(X_i;X_j)$ between the component random variables $X_i$ and $X_j$ for every $i\neq j \in [n]$. More generally in the non-realizable case, even if the distribution $P$ is not tree-structured, its closest tree-structured distribution $P^*$ in terms of the reverse-KL divergence to $P$,
\begin{align*}
    P^*=\arg\min\limits_{\mbox{tree structured} \ Q} \kl(P||Q),
\end{align*}
is also given by the same maximum spanning tree, where the reverse KL divergence is defined by $\kl(P||Q)=\int_{\R^n} P(x) \log \frac{P(x)}{Q(x)} dx$. A tree $\widehat{T}$ is said to be \emph{$\eps$-approximate tree} for $P$ if $\kl(P||P_{\widehat{T}}) \leq  \kl(P||P^*) + \eps$. A tree-structured distribution is said to be \emph{$T$-structured} if the underlying tree is $T$. In the \emph{structure learning problem}, we are interested in approximately recovering the best tree structure for the unknown distribution $P$ from random samples. Therefore, a natural approach to estimating the best tree structure is to estimate the pairwise mutual informations and simply return the estimated maximum spanning tree. This algorithm, which we refer to as the \emph{Chow-Liu algorithm}, has been extensively studied from asymptotic viewpoints where the goal is to accurately recover the underlying tree structure as the number of samples tends to infinity. In \cite{tan2010learningexp}, the authors studied the ``error exponent'' for the maximum-likelihood estimator of the tree structure, \cite{tan2010error}  studied the error exponents for composite hypothesis testing of Markov forest distributions, \cite{tan2010learninghypothesis} considered learning two tree graphical models to classify future observations into one of two categories, \cite{tan2010scaling} studied the problem of learning forests for discrete graphical models by removing edges from
the Chow–Liu tree. See the PhD thesis of \citep[Chapter 7]{jones2021bayesian} for several related discussions and references.

In contrast, a series of recent papers have looked at the question of \emph{prediction-centric learning}, where the goal is not to recover the exact tree structure but to only {\em approximately} recover the underlying tree structure such that the best distribution on the recovered tree structure is close to the true one in some distance measure such as KL divergence or variation distance, see \cite{bresler2020learning, boix2022chow, bhattacharyya2023near, DBLP:conf/stoc/DaskalakisP21,DBLP:journals/corr/abs-2310-06333} and the references therein.  
The idea is that such an approximate structure should be good enough for downstream inference tasks. Despite the simplicity and usefulness of tree-structured GGMs, a finite sample guarantee of approximately learning these models from samples has remained elusive so far.

Our current paper initiates the study of prediction-centric learning of Gaussian tree models in KL divergence and settles down its sample complexity using efficient algorithms.
Specifically, we characterize the minimum number of samples needed to learn an $\eps$-approximate tree in both realizable and non-realizable cases, up to small factors by giving almost tight sample-complexity upper and lower bounds. We show that the Chow-Liu algorithm always takes a near-optimal number of samples in terms of the approximation error ($\eps$) and the number of component random variables ($n$), to achieve a desired approximation guarantee for structure learning in both the realizable and non-realizable cases. Moreover, there is a quadratic gap for the optimal sample complexity of the realizable and the non-realizable cases: $\widetilde{\Theta}(n\eps^{-1})$ versus $\widetilde{\Theta}(n^2\eps^{-2})$ for any algorithm.

\subsection{Our results}

We will start with the result of additive estimation of the mutual information (MI) between two arbitrary-mean Gaussian random variables.  Given two arbitrary-mean Gaussian random variables $X$ and $Y$, we show that if we take $\Oh(1/\eps^2 + \log 1/\delta)$ samples, the empirical mutual information $\widehat{I}(X;Y)$ is $\eps$-close to the true mutual information $I(X;Y)$ with probability at least $1-\delta$. We also prove that the dependence on $\eps$ is tight. This result might be known in the community. However, as we will use this result for structure learning (to be discussed later), we give a formal proof for completeness.

\begin{restatable}{lem}{additiveUB}\label{cl:miestboundintro}
Let $X$ and $Y$ be two real-valued Gaussian random variables, and $\eps, \delta \in (0,1)$. If we take $\Oh(1/\eps^2 + \log 1/\delta)$ samples, with probability at least $1-\delta$, $\size{\widehat{I}(X;Y)-I(X;Y)} \leq \eps$ holds, where $\wh{I}(X;Y)$ is the empirical mutual information between $X$ and $Y$ obtained from~\Cref{alg:condmitester} with $m=\Oh(1/\eps^2 + \log 1/\delta)$. Moreover, the dependence of $\eps$ on the sample complexity is tight, and the time complexity is $\Oh(m)$.    
\end{restatable}

Now, instead of estimating the mutual information additively, we design an unconditional mutual information tester. It turns out that this problem is quadratically easier compared to the additive estimation problem. Here, given parameters $\eps, \delta \in (0,1)$, and two jointly distributed arbitrary-mean Gaussian $(X,Z)$, with probability at least $1-\delta$, our goal is to distinguish between the cases if $I(X;Z)=0$ or $I(X;Z) \geq \eps$. We show that the empirical mutual estimator is sufficient to distinguish between these two cases when taking $\Oh(1/\eps \log 1/\delta)$ samples. The formal statement of our result is as follows:

\begin{restatable}{lem}{mitestub}\label{lem:miubintro}
Let $(X,Z)$ be jointly distributed as an arbitrary-mean Gaussian. Fix any $\eps, \delta \in (0,1)$. Then the empirical mutual information estimator $\wh{I}(X;Z)$ satisfies the following guarantees with probability at least $(1-\delta)$ whenever $m\ge \Oh(\eps^{-1}\log \delta^{-1})$: (i) $I(X;Z) = 0$ implies $\wh{I}(X;Z) \leq \eps/20$, (ii) $I(X;Z) \geq \eps$ implies $\wh{I}(X;Z) \geq \eps/8$. Moreover, the time complexity is $\Oh(m)$.
\end{restatable}

Similar result also holds for conditional mutual information (CMI) tester. We show that $\Oh(1/\eps \log 1/\delta)$ samples are sufficient to distinguish $I(X;Y \mid Z) =0$ from $I(X;Y \mid Z) \geq \eps$. In fact, we prove that similar to the unconditional mutual information setting, empirical conditional mutual information tester is sufficient for this purpose.

\begin{restatable}{lem}{cmitestub}\label{lem:cmiubintro}
Let $(X,Y,Z)$ be jointly distributed as an arbitrary-mean Gaussian, and $\eps, \delta \in (0,1)$. Then the empirical conditional mutual information estimator satisfies the following with probability at least $(1-\delta)$ whenever $m \geq \Oh(1/\eps \log 1/\delta)$: (i) $I(X;Y \mid Z) = 0$ implies $\wh{I}(X;Y \mid Z) \leq \eps/20$, (ii) $I(X;Y \mid Z) \geq \eps$ implies $\wh{I}(X;Y \mid Z) \geq \eps/8$. Moreover, the time complexity is $\Oh(m)$.
\end{restatable}

Chain rule of conditional mutual information is a celebrated result in the machine learning and information theory community. However, here, we will be estimating and testing the (conditional) mutual information between Gaussian random variables by using regression. So, it is not immediately clear if the chain rule will continue to hold. Fortunately for us, we show that the chain rule of empirical mutual information remains true. We will later crucially use this result in the correctness proof of our realizable structure learning result.

\begin{restatable}{lem}{chainrulecmi}\label{lem:cmichainruleempintro}
 Let $\wh{I}(X;Z)$, $\wh{I}(X;Y),\wh{I}(X;Z),\wh{I}(X;Y \mid Z)$ and $\wh{I}(X;Z \mid Y)$ be our estimations from \Cref{alg:condmitester} using $m > 3$ samples. Then 
$\wh{I}(X;Z)+\wh{I}(X;Y \mid Z)=\wh{I}(X;Y)+\wh{I}(X;Z \mid Y)$.
\end{restatable}

We also show that our upper bound result on unconditional mutual information tester is tight, by showing a lower bound of $\Omega(1/\eps)$ samples for distinguishing between whether $I(X;Y) =0$ from $I(X;Y) \geq \eps$. Note that this also shows that our upper bound result for the conditional mutual information tester is tight. Moreover, any lower bound proven for zero-mean Gaussian random variables will continue to hold for arbitrary-mean Gaussian random variables.

\begin{restatable}{theo}{mitestlb}\label{theo:mitestlbintro}
Let $X, Y$ be two zero-mean Gaussian random variables. $\Omega(1/\eps)$ samples are necessary to distinguish if $I(X;Y)=0$ or $I(X;Y) \geq \eps$ with probability at least $2/3$.
\end{restatable}

We would like to note that we first study these problems under the assumption that the distributions are zero-mean Gaussian for the brevity of our calculations. However, there is nothing special about the zero-mean assumption, these results continue to hold for the case when the distributions have arbitrary-mean. Later, we present a reduction that translates the arbitrary mean assumption to the zero-mean assumption and requires only twice the number of samples as compared to the zero-mean setting.

Now we are ready to present our structure learning results. We start with the structure learning result for the non-realizable setting, where the unknown distribution $P$ from where we are obtaining samples may not be tree-structured. Our goal is to output an $\eps$-approximate tree $T$ of $P$ with probability at least $1-\delta$. We use an additive estimator of mutual information between two random variables where the goal is to estimate the mutual information of two Gaussian random variables, up to an additive error of $\eps$. This requires $\Oh(1/\eps^2)$ samples. We use this result along with a black-box reduction from \cite{bhattacharyya2023near} to prove that $\Oh(n^2/\eps^2 + \log n/\delta)$ samples are sufficient for outputting $\eps$-approximate tree for non-realizable distributions. 

\begin{restatable}{theo}{treelearningnonrealizable}\label{theo:treelearningnonrealizableintro}
Given sample access to an unknown $n$-variate arbitrary-mean Gaussian distribution $P$, with probability at least $1-\delta$, our algorithm takes $m=\Oh(n^2/\eps^2 + \log n/\delta)$ samples from $P$ and outputs an $\eps$-approximate tree $T$ of $P$ in $\Oh(mn^2)$ time.
\end{restatable}

Moreover, we show that the dependencies on $n$ and $\eps$ on the sample complexity are tight, ignoring poly-logarithmic factors. The associated hard instances of the lower bound are presented in the overview, and the formal proof is presented in \Cref{sec:nonrealizablelb}.


\begin{restatable}{theo}{treelearningnonrealizablelb}\label{theo:treelearningnonrealizablelbintro}
Given sample access to an unknown $n$-variate zero-mean Gaussian distribution $P$, with probability at least $9/10$, $\Omega(n^2/\eps^2)$ samples from $P$ are necessary to output an $\eps/200$-approximate tree $T$ of $P$.   
\end{restatable}

Now we proceed to present our result for the realizable setting where the unknown distribution $P$ from where we are obtaining samples is promised to be a tree-structured distribution. In contrast to the non-realizable setting, here we show that $\Oh(n/\eps \log n/\delta)$ samples from $P$ are sufficient for constructing an $\eps$-approximate tree $T$ of $P$. Note that this is quadratically better compared to the non-realizable setting. We achieve this advantage in sample complexity, because, unlike the non-realizable setting where we use an additive estimator of unconditional mutual information between a pair of Gaussian random variables, here we use the conditional mutual information tester as described in \Cref{lem:cmiubintro}. It is interesting to note that for our correctness analysis, we crucially use the chain rule of empirical conditional mutual information as stated in \Cref{lem:cmichainruleempintro}.
\begin{restatable}{theo}{treelearningrealizable}\label{theo:treelearningrealizableintro}
Given sample access to an unknown $n$-variate tree-structured arbitrary-mean Gaussian distribution $P$, with probability at least $1-\delta$, our algorithm takes $m=\Oh(n/\eps \log n/\delta)$ samples from $P$ and outputs an $\eps$-approximate tree $T$  in $\Oh(mn^2)$ time. 
\end{restatable}
Additionally, we show that the dependencies of sample complexity on $n$ and $\eps$ are tight.

\begin{restatable}{theo}{treelearningrealizablelb}\label{theo:treelearningrealizablelbintro}
Given sample access to an unknown $n$-variate tree-structured zero-mean Gaussian distribution $P$, with probability at least $9/10$, $\Omega(n/\eps)$ samples from $P$ are necessary to output an $\eps/320$-approximate tree $T$.   
\end{restatable}

As with the non-realizable case, we present the hard instances for \Cref{theo:treelearningrealizablelbintro} in the overview, while the formal proof is in \Cref{sec:realizablelb}.
In fact, all our structure learning lower bound results follow the same recipe, where we use coding theoretic arguments, along with Fano's inequality (see \Cref{lem:fanoinequality}). On the other hand, for proving lower bounds on unconditional and conditional mutual information testers, we use Le Cam's method (\Cref{lem:lecam}).  Similar to our results on unconditional and conditional mutual information testers, we prove \Cref{theo:treelearningnonrealizableintro} and \Cref{theo:treelearningrealizableintro} under the assumption that the distributions are zero-mean and later design a reduction from the arbitrary-mean setting to the zero-mean setting.

We note that given $\tOh(n/\eps)$ samples from $P$ and a tree $T$, one can recover a $T$-structured distribution $Q$ such that with probability at least $1-\delta$, $\kl(P||Q)- \kl(P||P_T) \leq \eps$ holds, where $P_T$ is the best $T$-structured distribution for $P$. This problem is known as the \emph{distribution recovery} problem and the result follows from the result of \citep[Theorem 3.1]{DBLP:conf/aistats/0001CGGW22}. This algorithm can be daisy-chained with our structure recovery algorithms to get a distribution $Q$ such that $\kl(P||Q) \leq \kl(P||P^*) + \eps$, using $\tOh(n/\eps)$ samples with high probability, where $P^*$ is the best tree-structured distribution for $P$.
In particular, when $P$ is itself tree-structured, it gives us an algorithm for sample-optimally learning a tree-structured distribution in variation distance using Pinkser's inequality (see \Cref{sec:realizable}).

Note that, in contrast to the discrete setting, it is always possible to efficiently learn an $n$-variate Gaussian in reverse-KL divergence distance at most $\eps$ using $O(n^2\eps^{-1})$ samples~\cite{zhang2024properties,ashtiani2020near,bhattacharyya2022learning} and this is sample-optimal. However, our work shows that the sample complexity of learning an $\eps$-approximate tree is a $\Theta(\eps^{-1})$-factor worse in general and is $\Theta(n)$-factor better for tree-structured distributions.

\subsection{Additional Related works}
As mentioned in the introduction, learning high dimensional distributions from samples is intractable in general (see \cite{kamath2015learning}). As a result, several structural assumptions are made to design efficient algorithms for learning high dimensional distributions such as Bayesian networks, Markov Random Fields etc. Unfortunately, even in these settings, learning the best structure is known to be NP-hard (\cite{chickering1996learning,meek2001finding}). Interestingly, Chow-Liu is one of the very few algorithms that are efficient. As a result, there have been several works to understand the applicability of this algorithm in various settings such as in \emph{bounded treewidth graphs}: \cite{srebro2003maximum,narasimhan2012pac}, \emph{mixture of trees}: \cite{meila2000learning,anandkumar2012learning}, \emph{polytrees}: \cite{dasgupta2013learning} etc. 
However, these works mostly focus on the asymptotic regime of the sample complexity. Only recently, researchers have been looking into the performance of these algorithms in the non-asymptotic regime. \cite{bhattacharyya2023near} and \cite{DBLP:conf/stoc/DaskalakisP21} were the one of the first few works to study the non-asymptotic behavior of Chow-Liu algorithm. Very recently \cite{DBLP:journals/corr/abs-2310-06333} studied the problem of learning polytree when the underlying skeleton is known.
There have also been many interesting recent works on Ising model: \cite{bresler2015efficiently,bresler2020learning,boix2022chow,kandiros2023learning,DBLP:conf/stoc/DaskalakisP21,gaitonde2023unified,klivans2017learning}.

One of the crucial components of our algorithm for structure learning is a novel conditional mutual information tester of Gaussian random variables, which is closely related to the independence testing of two Gaussian random variables. On a general note, this problem falls in the purview of the area of \emph{distribution testing}, where the goal is to infer some property of an unknown probability distribution, with only sample access from it. See the surveys~\cite{rubinfeld2012taming,Canonne:Survey:ToC,CanonneTopicsDT2022} for an introduction to this field. Formally, given sample access to an unknown joint distribution $D$ on variables $(X,Y)$, the goal is to distinguish with probability at least $2/3$, if $X$ and $Y$ are independent or far from being independent under some suitable distance measure. This problem has been extensively studied in the distribution testing community, see \cite{batu2001testing, acharya2015optimal,diakonikolas2021optimal}. Recently, the conditional variant of independence testing has also been studied (\cite{canonne2018testing,neykov2021minimax, kim2022local, marx2019testing, kim2023conditional}). There have also been some works for Bayesian networks (\cite{DBLP:conf/colt/CanonneDKS17,DBLP:journals/corr/abs-2204-08690}).

However, all these works are mostly focused on the case when the unknown distribution is discrete. There have been very few works for the continuous setting: see for example: \cite{shah2020hardness,neykov2021minimax}.
However, most of these works are studied in the asymptotic setting, namely the performance of the algorithms when the number of samples tends to infinity. One important work in this scenario is \cite{kelner2020learning}, where the authors studied the problem of exact structure learning (not necessarily tree) and gave certain parameterized sample-complexity bounds.
In this work, we study these problems in the non-asymptotic finite sample regime and make significant progress in this direction for unconditionally recovering an approximate tree structure.

\paragraph{Deviation from~\cite{bhattacharyya2023near}}
As mentioned before, the work of \cite{bhattacharyya2023near} also follows a similar line, where the idea is to reduce the structure learning problem to conditional mutual information testing as reported by \cite{bhattacharyya2023near}. However, there are several crucial differences between this work and \cite{bhattacharyya2023near}. We will now briefly discuss them.
\begin{enumerate}[wide, labelwidth=!, labelindent=0pt]
    \item In \cite{bhattacharyya2023near}, the authors studied the problem when the distributions were discrete. However, we are considering Gaussian distributions, so their techniques do not immediately translate here.

    \item In the correctness proof of the conditional mutual information tester, the authors in \cite{bhattacharyya2023near} showed that if $I(X;Y \mid Z)\geq \eps$, then $\wh{I}(X;Y \mid Z) \geq C \cdot I(X;Y \mid Z)$ for some absolute constant $C$. 
    Instead, here we show an alternative guarantee that if $I(X;Y \mid Z)\geq \eps$, then $\wh{I}(X;Y \mid Z) \geq C \cdot \eps$, for some suitable constant $C$, which is sufficient for our analysis. This requires some modifications in the structure learning proof of \cite{bhattacharyya2023near}, which we present in \Cref{sec:cmiub}.

    \item The chain rule of mutual information is crucially used for the correctness analysis for the structure learning problem in the realizable scenario in \cite{bhattacharyya2023near}. In their case, they estimated the (conditional) mutual information as the (conditional) mutual information of the observed distribution. Therefore the chain rule was immediate for them. 
    However, in this work, we are using regression to estimate the unconditional and conditional mutual information.
    As a result, it is not immediately clear that the chain rule will continue to hold for the empirical conditional mutual information as well.    
    One of the crucial contributions of this work is to show that similar results continue to hold in our setting.

    \item The authors in \cite{bhattacharyya2023near} designed a conditional mutual information tester using averaging arguments. However, this technique can not be extended for the Gaussian setting since the probability density at a particular point is $0$. In this work, we designed a novel tester for this problem using regression on a suitably transformed data.

    \item The sample complexity of the conditional mutual information tester for the discrete setting is $\Omega(\eps^{-1}\log \eps^{-1})$, as shown in \cite{bhattacharyya2023near} for constant $\delta$. However, interestingly, here in the Gaussian setting, our conditional mutual information tester requires only $\Oh(\eps^{-1})$ samples for constant $\delta$. So, there is a logarithmic improvement in the sample complexity as compared to \cite{bhattacharyya2023near}. Similar improvements also hold for structure learning. 

    \item Our lower bounds are technically different from that of \cite{bhattacharyya2023near}. In a broad sense, they use Le Cam's method for their proofs, whereas we use Fano's inequality along with coding theoretic arguments to construct the hard instances of our lower bounds.
 
\end{enumerate}

\vspace{-9.5pt}

\paragraph{Comparison with \cite{DBLP:conf/aistats/WangGTA024}:} 
In a parallel work, Wang et al.~\cite{DBLP:conf/aistats/WangGTA024} have looked at the problem of Gaussian polytree learning. They have independently established the same optimal sample complexity bounds for tree structure learning by the Chow-Liu algorithm in both realizable and agnostic cases. However, there are some important difference between the two papers that are worth highlighting. Their conditional mutual information tester and its analysis directly replaces certain suitable regression coefficients with its empirical counterpart. Whereas, our mutual information tester is based on the regression between some suitably transformed data. We also give a detailed proof of the empirical chain rule of Gaussian mutual information which is not straightforward and is missing in their paper. Finally, the experiments across the two papers are very different. We compare the Chow-Liu algorithm with GLASSO and CLIME algorithms for structure recovery tasks on the hard instances given by our lower bounds; whereas they compare Chow-Liu algorithm with PC and GES algorithms for random trees. Additionally, we perform experiments to confirm the theoretical convergence rates of the mutual information testers on the hard instances given by our lower bounds. Their paper also gives a result for learning Gaussian polytrees.

\subsection*{Organization of the paper}
The rest of the paper is organized as follows. In \Cref{sec:prelim}, we present the preliminaries, followed by overviews of our results in \Cref{sec:overview}. In \Cref{sec:nonzeromean}, we start by showing that all our testing results continue to work for arbitrary mean Gaussian random variables by presenting a reduction to the zero-mean setting. 
Then we present our proof of the additive estimator of mutual information between two Gaussian random variables (\Cref{cl:miestboundintro}) in \Cref{sec:addtivemiest}, which we also prove is tight. Then in \Cref{sec:miub}, we give the proof of our mutual information tester for two Gaussian random variables (\Cref{lem:miubintro}), followed by the proof of the conditional mutual information tester (\Cref{lem:cmiubintro}) in \Cref{sec:cmiub}. We prove the chain rule of empirical conditional mutual information (\Cref{lem:cmichainruleempintro}) in \Cref{sec:cmichainrule}.
In \Cref{sec:mitestlb}, we present our result of the tight lower bound of mutual information tester (\Cref{theo:mitestlbintro}). 

Later, in \Cref{sec:structurelearnnonzeromean}, we show reductions which prove that our structure learning results continue to hold for non-zero-mean setting. In \Cref{sec:nonrealizable}, we prove our structure learning result for non-realizable setting (\Cref{theo:treelearningnonrealizableintro}), followed by the lower bound in the non-realizable setting (\Cref{theo:treelearningnonrealizablelbintro}). Then, in \Cref{sec:nonrealizablelb}, we present the proof of our structure learning result for realizable setting (\Cref{theo:treelearningrealizableintro}), followed by the associated lower bound (\Cref{theo:treelearningrealizablelbintro}) in \Cref{sec:realizablelb}. In \Cref{sec:experiments_full}, we present our experimental results, and conclude in \Cref{sec:conclusion}.

\section{Preliminaries}\label{sec:prelim}
In this paper, we will be using the following notations. we will use $[n]$ to denote the set $\{1, \ldots,n\}$. All logarithms considered here are natural.
For concise expressions and readability, we use the asymptotic complexity notion of $\widetilde{\Oh}$, where we hide poly-logarithmic dependencies of the parameters. 

Let us start with some definitions.

\begin{defi}[Positive semi-definite matrix]
Let $M$ be a symmetric matrix of dimension  $n \times n$. $M$ is said to be \emph{positive semi-definite} (PSD in short) if $x^TMx \geq 0$ holds for every vector $x \in \R^n$.  
\end{defi}

\begin{obs}[Difference between determinants of positive semi-definite matrices]
Let $A$ and $B$ be two positive semi-definite matrices of dimension $n \times n$ such that $A-B$ is also a positive semi-definite matrix. Then $\det(A) \geq \det(B)$.   
\end{obs}

Let us now define the notion of sub-exponential random variables.

\begin{defi}[Sub-exponential random variable]
A random variable $X$ is said to be \emph{sub-exponential} if the moment generating function of $\abs{X}$ satisfies:
\[
\E [\exp\paren{|X|/{K}}] \le 2,
\]
for some $K$. The smallest such $K$ is called the sub-exponential norm of $X$, denoted $\abs{\abs{X}}_{\psi}$.
\end{defi}

\begin{fact}\label{fact:sqgauss}
    Let $g\sim N(0,1)$  be a standard Gaussian random variable. Then $\abs{\abs{g^2}}_{\psi}\le \frac{8}{3}$.
\end{fact}

\begin{lem}[Bernstein inequality, Theorem $2.8.4$ of \cite{vershynin2018high}]\label{lem:bernsteain}
Let $X_1, \ldots, X_N$ a set of independent, zero-mean random variables, such that $|X_i| \leq K$  for all $i \in [N]$ for some $K$. Moreover, let $\sigma^2 = \sum_{i=1}^n \E[X_i^2]$ be the variance of the sum of the random variables. Then, for
every $t \geq 0$, we have:
$$\Pr\left[\size{\sum_{i=1}^N X_i} > t\right] \leq 2 \exp\left(-\frac{t^2/2}{\sigma^2+ K\cdot t/3}\right)$$
\end{lem}

\begin{defi}[KL-divergence between two distributions]
Let $P$ and $Q$ be two distributions defined over $\R^n$. Then the \emph{$\mathsf{KL}$-divergence} between $P$ and $Q$ is defined as: 
$\kl(P||Q)= \int_{x \in \R^n} P(x) \log \frac{P(x)}{Q(x)} dx$, where $P(x)$ and $Q(x)$ denote the pdfs of $P$ and $Q$, respectively.

\end{defi}


\begin{defi}[Total variation distance]
Let $P$ and $Q$ be two probability distributions over $\R^n$. The total variation distance between $P$ and $Q$ is:
$$\mathsf{d_{TV}}(P,Q)=\frac{1}{2}\int_{\R^n}|dP-dQ|.$$

\end{defi}

\begin{lem}[Pinkser's inequality]\label{lem:pinkser}
Let $P$ and $Q$ be two probability distributions. Then the following holds:
$$\mathsf{d_{TV}}(P,Q) \leq \sqrt{\frac{\kl(P||Q)}{2}}.$$

\end{lem}

\begin{obs}\label{obs:klhellingerbounds}
Let $P$ and $Q$ be two zero-mean Gaussian distributions such that $P \sim N(0, \Sigma_1)$ and $Q \sim N(0,\Sigma_2)$. Then the $KL$-divergence between $P$ and $Q$ can be expressed as follows~\citep[Page 13]{duchi2007derivations}:
    \begin{equation*}
     \kl(P || Q) = \frac{1}{2} \left(\tr(\Sigma_2^{-1} \Sigma_1) - \mathsf{dim}(\Sigma_1) + \ln\left(\frac{\det(\Sigma_2)}{\det(\Sigma_1)}\right)\right)   
    \end{equation*} 
\end{obs}

\begin{defi}[Differential Entropy (in nats)]
Let $X=(X_1,\dots,X_n)\sim P=N(\mu,\Sigma)$ be an $n$-variate Gaussian. 
Then its differential entropy is defined as
$H(X):=\E_{X\sim P} [{-\ln P(X)}]$.
\end{defi}

\begin{defi}[Gaussian mutual information (in nats)]
Let $X=(X_1,\dots,X_n)\sim P=N(\mu,\Sigma)$ be an $n$-variate Gaussian. 

Let $X_S:=\braces{X_i: i\in S}$ denote the joint random variable restricted to $S\subseteq [n]$. Then for any $S,T\subseteq [n]$, the mutual information between $X_S$ and $X_T$ is defines as:
\[
I(X_S;X_T)=H(X_S)+H(X_T)-H(X_{S\cup T}).
\]

For any $R\subseteq [n]$, the conditional mutual information between $X_S$ and $X_T$ conditioned on $X_R=x_R$ for a fixed $x_R\in \mathbb{R}^{|R|}$, denoted $I(X_S;X_T\mid X_R=x_R)$, is defined as the mutual information between $X_S$ and $X_T$ in the conditional distribution of $X$ given $X_R=x_R$. 

Finally, the the conditional mutual information between $X_S$ and $X_T$ conditioned on $X_R$, denoted $I(X_S;X_T\mid X_R)$, is defined as the expectation of $I(X_S;X_T\mid X_R=x_R)$ with respect to $X_R$ as $X_R$ varies according to the marginal distribution of $P$ on $X_R$:
\[
I(X_S;X_T\mid X_R)=\E_{X_R\sim P} [I(X_S;X_T\mid X_R=x_R)]
\]
\end{defi}

\begin{fact}[Gaussian Mutual Information formula]\label{fact:gaussianMI}
Let $X=(X_1,\dots,X_n)\sim N(\mu,\Sigma)$ be an $n$-variate Gaussian. Let $X_S:=\braces{X_i: i\in S}$ denote the joint random variable restricted to $S\subseteq [n]$. Let $M_{S}$ be the submatrix of $\Sigma$ corresponding to the random variables $X_S$. Then for any $S,T\subseteq [n]$,
\[
I(X_S;X_T)=\frac{1}{2}\log \paren{\frac{\det\paren{M_S}\det\paren{M_T}}{\det\paren{M_{S\cup T}}}
},
\]
\[
I(X_S;X_T\mid X_R)=I(X_S;X_{R\cup T})-I(S_X;X_R).
\]

In particular, note that the (conditional) mutual information is invariant to a scaling of $\Sigma$ by some constant factor.
\end{fact}

We will use the following chain rule of mutual information.

\begin{obs}\label{obs:michainrule}
Let $X, Y$, $Z$ be $3$ random variables. Then
$I(X;Y)-I(Y;Z) = I(X;Y \mid Z) - I(X;Z \mid Y)$ holds.
\end{obs}

\begin{lem}[Fano's inequality: \cite{yu1997cam}]\label{lem:fanoinequality}
Let $\mathcal{D}$ be a class of distributions, $\abs{\mathcal{D}}\ge 3$. Given iid samples from a random distribution $D$ from $\mathcal{D}$, $\Omega\left(\frac{\log \abs{\mathcal{D}}}{\max\limits_{D_i,D_j\in \mathcal{D}} \kl\paren{D_i||D_j}}\right)$ samples are necessary to correctly guess $D$ with probability $9/10$.
\end{lem}

\begin{lem}[Le Cam's two point method: \cite{yu1997cam}]\label{lem:lecam}
    Let $P$ and $Q$ be two distributions. Given a random unknown distribution $R\in \{P,Q\}$, any algorithm that takes a single sample from $R$ will be incorrect with at least $\paren{\frac{1}{2}-\frac{\mathsf{d_{TV}}(P,Q)}{2}}$ probability in guessing $R$.
\end{lem}

The following folklore result 
states that we can always express a multi-variate Gaussian as a linear system of equations with Gaussian noise.

\begin{obs}[\cite{koller2009probabilistic}]\label{obs:zeromeangaussianstructure}
Let $(X, Y, Z)$ be jointly distributed as a zero-mean Gaussian. Then there exists a set of parameters $a,b,c,\alpha, \beta, \gamma \in \R$ such that $X,Y,Z $ follow the following linear equations: (i) $Z \sim N(0,a^2)$, (ii) $V \sim N(0,b^2)$, (iii) $U \sim N(0,c^2)$, (iv) $X \longleftarrow \alpha Z + V$, (v) $Y \longleftarrow \beta X + \gamma Z + U$.

\end{obs}

For the brevity of the arguments, we will be using the above representation throughout this work.

\begin{obs}\label{obs:mibasic}
    Let $(X,Y,Z)$ be a $3$-variate Gaussian random variable with means $\paren{
    \mu_X,\mu_Y,\mu_Z
    }
    $ such that the following linear system of equations holds for some $a,b,c,\alpha,\beta, \gamma$:
    (i) $\paren{Z-\mu_Z}\sim N(0,a^2)$ 
    (ii) $V\sim N(0,b^2)$ 
    (iii) $U\sim N(0,c^2)$
    (iv) $\paren{X-\mu_X}\longleftarrow \alpha \paren{Z-\mu_Z} + V$
    (v) $\paren{Y-\mu_Y} \longleftarrow \beta \paren{X-\mu_X}+\gamma\paren{
    Z-\mu_Z
    } + U$. Fix any $z\in \mathbb{R}$. Then the following equations hold:
    \begin{align*}
        &I(X,Y\mid Z)=I(X;Y\mid Z=z)=\frac{1}{2}
    \log \paren{1+\beta^2b^2c^{-2}},\\
    &I(X;Y)=\frac{1}{2}\log \frac{\paren{\alpha^2a^2+b^2}\paren{
    \paren{\alpha\beta+\gamma}^2+\beta^2b^2+c^2
    }}{a^2b^2\gamma^2+c^2\paren{\alpha^2a^2+b^2}}.
    \end{align*}
\end{obs}

Note that the mutual information is invariant to mean translation. In particular, $I(X;Y\mid Z=z)$ does not depend on $z$, although the latter is different from $I(X;Y)$. Also, $I(X;Y\mid Z)=0$ whenever $\beta=0$ (equivalently the edge $(X,Y)$ is missing).

\begin{obs}[Folklore, see \citep{kelner2020learning}]\label{obs:empcovmatrixbound}
Let $P$ be an unknown zero-mean Gaussian distribution defined over $\R^n$ and let $\eps, \delta \in (0,1)$ be two parameters. Given $m$ i.i.d. random vectors $X^{(1)}, \ldots, X^{(m)}$ from $P$, let $\widehat{\Sigma}$ be the empirical covariance matrix. Then if {$m=\Oh\paren{\paren{n+\log \delta^{-1}}\eps^{-2}}$} with probability at least $1-\delta$, the following holds:
$$(1-\eps)\Sigma \preceq \widehat{\Sigma} \preceq (1+\eps)\Sigma$$
\end{obs}




Now let us proceed to define tree-structured distributions. Let $T$ be an undirected tree on the vertex set $[n]$. Consider a vertex $i \in [n]$ as the root of $T$. Then the rooted orientation of $T$ orients the edges from the root $i$. Given a node $i$, $pa(i)$ denotes the parent node of $i$, or null if $i$ is the root node, and $nd(i)$ denotes the set of vertices not reachable from $i$.

\begin{defi}[Tree-structured distribution]
Let $T$ be a tree on the vertex set $[n]$, and consider a rooted-orientation of $T$, and order the nodes of $T$ in the topological order. A distribution $P$ over $X=(X_1, \ldots,X_n) \in \R^n$ is said to be a \emph{$T$-structured distribution} if every variable $X_i$ is conditionally independent of $\{X_j : j \in nd(i)\}$ given $X_{pa(i)}$. Moreover, the density function of $P$ has the following factorization:
$$P[X=x]:= P[X_1=x_1] \cdot \prod P[X_i=x_i \mid X_{pa(i)}=x_{pa(i)}]$$
A distribution $P$ is said to be \emph{tree-structured} if it is $T$-structured for some tree $T$.
\end{defi}

\begin{lem}[\cite{verma2022equivalence}]\label{lem:treeci}
Let us consider a tree $T$ with the vertex set $[n]$. Moreover, consider a $T$-structured distribution $P$ on $(X_1, \ldots,X_n)$. For any three nodes $i,j,k \in [n]$, if the path from $i$ to $k$ passes through $j$ in $T$, then  $X_i$ and $X_k$ are conditionally independent given $X_j$.
\end{lem}

\begin{lem}[\cite{bhattacharyya2023near}]\label{lem:spanningtreelemma}
Let $T_1$ and $T_2$ be two spanning trees on the vertex set $[n]$. Consider the symmetric difference of $T_1$ and $T_2$: $F=\{f_1, \ldots, f_{\ell}\}$ and $G=\{g_1, \ldots,g_{\ell}\}$. Then $F$ and $G$ can be paired up, say $\langle f_i, g_i \rangle$ such that $T_1 \cup \{f_i\} \setminus \{g_i\}$ is a spanning tree for all $i \in [\ell]$.   
\end{lem}

\begin{defi}[$\eps$-approximate tree]
Consider a distribution $P$ defined over $\R^n$. A tree $T$ on the vertex set $[n]$ is said to be \emph{$\eps$-approximate} if there exists a $T$-structured distribution $Q_T$ such that the following holds:
$$\kl(P||Q_T) \leq \eps + \min_{\mbox{tree} \ T'} \min_{T'-\mbox{structured dist.} \ Q'} \kl(P||Q')$$
\end{defi}

\begin{theo}[\cite{chow1968approximating}]\label{theo:chowliualgo}
Let $P$ be a distribution over $\R^n$, and $T$ be a tree over the vertex set $[n]$. Let $P_T$ be the best distribution of $P$ for the tree $T$. Then we have:
$\kl(P||P_T) = J_P -\mathsf{wt}_P(T),$
where $J_P=\sum_v H(P_v) - H(P)$, a quantity independent of $T$, and $\mathsf{wt}_P(T)= \sum_{(X,Y) \in T}I(X;Y)$.
\end{theo}

\begin{algorithm}[ht]
\caption{Chow-Liu Algorithm}
\label{alg:chowliu}


\For{$1\leq i < j \leq n$}{

$\widehat{I}(X;Y) \gets$ estimated mutual information between the nodes $X,Y$, estimated using~\Cref{alg:condmitester}. \

}

$G \gets$ complete weighted graph on the vertex set $[n]$, where the weight of every edge $(X,Y)$ is $\widehat{I}(X;Y)$. 

$T \gets$ a maximum spanning tree of $G$. \

\Return T. \
                           
\end{algorithm}

\section{Overview of our results}\label{sec:overview}
In this section, we present overviews of our results. We start with our results for estimating and testing the mutual and conditional mutual information, and then proceed to the structure learning results for both realizable and non-realizable settings. For the brevity of the calculations, we will be working with the zero-mean Gaussian random variables, and finally extend them to arbitrary-mean setting using reductions, which require only twice the number of samples as that of the zero-mean setting.  We present a general-purpose algorithm (\Cref{alg:condmitester}) for (conditional) mutual information (MI) estimator and tester for the zero-mean setting, called with a suitable number of samples ($m$). The formal proofs are presented in \Cref{sec:mitesterproof}.

\paragraph{Description of \Cref{alg:condmitester}}
Given $m$ samples depending upon the problem, \Cref{alg:condmitester} first computes the correlation coefficients (either $\wh{\rho}_{X,Z}$ or $\wh{\rho}_{\widetilde{X}, \widetilde{Y}}$), and then estimates $\wh{I}(X;Z)$ and $\wh{I}(X;Y \mid Z)$. For MI estimation problem, it returns $\wh{I}(X;Z)$, and for the testing of (conditional) MI problems, it suitably checks the values and decides accordingly.

\subsection{Overview of mutual information testers}

\paragraph{Additive estimation of mutual information}
We start with our result of additive estimation of mutual information. Refer to the algorithm presented in \Cref{alg:condmitester} and the notations therein, called with $m=\Oh(1/\eps^2 + \log 1/\delta)$. Given two zero-mean Gaussian random variables $X$ and $Y$, we show that if we take $\Oh(1/\eps^2 + \log 1/\delta)$ samples, the empirical mutual information $\widehat{I}(X;Y)$ is $\eps$-close to the true mutual information $I(X;Y)$ with probability at least $1-\delta$. We first estimate the correlation coefficient between $X$ and $Y$,
\[\widehat{\rho}_{X,Y}:=\frac{\paren{\bar{X}\cdot\bar{Y}}}{\sqrt{\bar{X}\cdot\bar{X}}\sqrt{\bar{Y}\cdot\bar{Y}}}\]
and then we estimate $\widehat{I}(X;Y)=-1/2 \log (1- \widehat{\rho}_{X,Y}^2)$.  We also show that the dependence on $\eps$ for additive estimation is tight.
We now present the hard instances for our lower bound. Let us assume $ U_0, V_0, U_1, V_1 \sim N(0,1)$ be four iid random bits. The null hypothesis $H_0$ is defined as: (i) $X_0 \gets U_0$, (ii) $Y_0 \gets (1/2+ \eps) \cdot X_0 + V_0$. The alternate hypothesis $H_1$ is defined as: (i) $X_1 \gets U_1 $, (ii) $Y_1 \gets (1/2- \eps) \cdot X_1 + V_1$.

\paragraph{Mutual information tester}
Now we present our result for the mutual information tester.
Given a zero-mean Gaussian joint distribution $(X,Z)$ and $\eps \in (0,1)$,
our goal here is to distinguish between if $I(X;Z)=0$ or $I(X;Z) \geq \eps$. Following \Cref{obs:zeromeangaussianstructure}, this is equivalent to consider the structure where $Z \sim N(0,a^2), V \sim N(0,b^2)$ and $X \gets \alpha Z+ V$. Suppose we have taken $m$ pairs of samples $(X_1,Z_1), \ldots, (X_m,Z_m)$ where $m=\Oh(1/\eps \log 1/\delta)$. We use \emph{Ordinary Least Square (OLS)} algorithm to estimate 
$\widehat{a}$ and $\widehat{\alpha}$, the estimates of $a$ and $\alpha$, respectively. Then, we estimate $\Var[\widehat{Z}]=\frac{1}{m}\paren{\bar{Z}\cdot\bar{Z}}$, $\Var[\widehat{X}]=\frac{1}{m}\paren{\bar{X}\cdot\bar{X}}$, as well as $\Cov[\widehat{X},\widehat{Z}]=\frac{1}{m}\paren{\bar{X}\cdot\bar{Z}}$. Once we have $\Var[\widehat{X}], \Var[\widehat{Z}]$, $\Cov[\widehat{X},\widehat{Z}]$, we estimate the correlation coefficient $\widehat{\rho}_{XZ}= \Cov[\widehat{X},\widehat{Z}]/\sqrt{\Var[\widehat{X}] \Var[\widehat{Z}]}$. We show that $\wh{\rho}_{X,Z}=\rho_{X,Z}\pm\Theta(\sqrt{\eps})$ with high probability using $\Oh(\eps^{-1} \log \delta^{-1})$ samples from Bernstein's inequality.
Since $I(X;Z)= -1/2 \log(1-\rho_{X,Z}^2)$ and $\wh{I}(X;Z):= -1/2 \log(1-\wh{\rho}_{X,Z}^2)$, to distinguish if $I(X;Z)=0$ or $I(X;Z) \geq \eps$, it is sufficient to test if $\widehat{I}({X;Z}) \geq \frac{1}{8}{\eps}$ or not. The algorithm is presented in~\Cref{alg:condmitester}, called with $m=\Oh(1/\eps \log 1/\delta)$.

\SetKwComment{Comment}{/* }{ */}


    \begin{algorithm}
\caption{Empirical-CMI$(m)$}\label{alg:condmitester}
Get $m$ samples $\bar{X}=(X_1,\dots,X_m), \bar{Y}=(Y_1,\dots,Y_m), \bar{Z}=(Z_1,\dots,Z_m)$.


\Comment{Unconditional MI tester}

$\wh{\rho}_{X,Z}\gets\frac{\paren{\bar{X}\cdot\bar{Z}}}{\sqrt{
    \paren{\bar{X}\cdot\bar{X}}\paren{\bar{Z}\cdot\bar{Z}}}}$

$\wh{I}(X;Z) \gets -\frac{1}{2}\log\paren{1 - {\wh{\rho}_{X,Z}}^2}$


\If{$\widehat{I}(X,Z) \geq \eps/8$}{
Output  $I(X;Z) \geq \eps$. \
}

\Else{

Output $I(X;Z)=0$. \ 

}

\Comment{Conditional MI tester }

$\wh{\alpha} \gets \frac{\paren{\bar{X}\cdot\bar{Z}}}{\paren{\bar{Z}\cdot\bar{Z}}}$, $\wh\beta \gets \frac{\paren{\bar{X}\cdot\bar{Y}}\paren{\bar{Z}\cdot\bar{Z}} - \paren{\bar{X}\cdot\bar{Z}}\paren{\bar{Y}\cdot\bar{Z}}}{\paren{\bar{X}\cdot\bar{X}}\paren{\bar{Z}\cdot\bar{Z}} - \paren{\bar{X}\cdot\bar{Z}}^2}$

$\wh\gamma \gets\frac{\paren{\bar{Y}\cdot\bar{Z}}\paren{\bar{X}\cdot\bar{X}} - \paren{\bar{X}\cdot\bar{Z}}\paren{\bar{X}\cdot\bar{Y}}}{\paren{\bar{X}\cdot\bar{X}}\paren{\bar{Z}\cdot\bar{Z}} - \paren{\bar{X}\cdot\bar{Z}}^2}$

$\widetilde{X}\gets \bar{X}-\widehat{\alpha} \bar{Z}$, $\widetilde{Y}\gets  \bar{Y}-(\widehat{\alpha} \widehat{\beta} + \widehat{\gamma}) \bar{Z}$

$\wh{\rho}_{\wt{X},\wt{Y}} \gets 
        \frac{
        \paren{\wt{X}\cdot\wt{Y}}
        }{
        \sqrt{
        \paren{\wt{X}\cdot\wt{X}}
        \paren{\wt{Y}\cdot\wt{Y}}
        }
        }$


$\wh{I}(X;Y \mid Z) \gets -\frac{1}{2}\log\paren{1 - \wh{\rho}^2_{\wt{X},\wt{Y}}}$

\If{$\wh{I}(X;Y \mid Z) \geq \eps/8$}{
Output  $I(X;Y|Z) \geq \eps$. \

} 



\Else{

Output $I(X;Y|Z)=0$. \

}





                           
\end{algorithm}

\paragraph{Conditional mutual information tester}
Now we discuss our result of the conditional mutual information testing of three zero-mean Gaussian random variables $X,Y,Z$. More formally, given  $\eps \in (0,1)$, we want to distinguish between $I(X;Y \mid Z)=0$ from $I(X;Y \mid Z) \geq \eps$. One crucial ingredient of our analysis is the fact that $I(X;Y | Z=z)= I(X;Y | Z=0)$.
Thus, from \Cref{obs:zeromeangaussianstructure}, we consider the following linear system, where $Z \sim N(0,a^2), V \sim N(0,b^2), U \sim N(0,c^2)$, $X \gets \alpha Z + V$ and $Y \gets \beta X + \gamma Z + U$. Thus, to distinguish $I(X;Y \mid Z)=0$ from $I(X;Y \mid Z) \geq \eps$, it is sufficient to estimate the correlation coefficient between 
$X- \alpha Z$ and $Y- (\alpha \beta + \gamma) Z$ and threshold it appropriately as in the case of the unconditional tester discussed earlier. However, as we do not know the parameters $\alpha, \beta$ and $\gamma$, we again apply regression techniques to first estimate $\widehat{\alpha}, \widehat{\beta}$ and $\widehat{\gamma}$, and then estimate the correlation coefficient between $X- \widehat{\alpha} Z$ and $Y - (\widehat{\alpha} \widehat{\beta}+ \widehat{\gamma})Z$. We also prove that the total error in this procedure is not too much again using Bernstein's concentration inequality. This requires $\Oh(1/\eps \log 1/\delta)$ samples in total and is presented in~\Cref{alg:condmitester}, called with $m=\Oh(1/\eps \log 1/\delta)$.

\paragraph{Chain rule of empirical conditional mutual information}
Now we discuss another crucial ingredient of our technical results: the chain rule of empirical conditional mutual information. Given three random variables $X,Y$, and $Z$, it is a very well-known and useful fact that
$$I(X;Y) + I(X;Z \mid Y) = I(X;Z) + I(X;Y \mid Z).$$ 
However, as we will be estimating the (conditional) mutual information $\widehat{I}(X;Y), \widehat{I}(X;Z)$, as well as $\widehat{I}(X;Y \mid Z)$ and $\widehat{I}(X;Z \mid Y)$ using regression, it is not immediately clear if the chain rule continues to hold for the estimated mutual and conditional mutual information between $X,Y$ and $Z$. We show that the chain rule remains true in this scenario, that is:
$$\wh{I}(X;Z)+\wh{I}(X;Y \mid Z) =\wh{I}(X;Y)+\wh{I}(X;Z \mid Y)$$
Specifically, we show that our MI and CMI estimators can be seen as that of a linear Gaussian model whose coefficients are such that its covariance matrix coincides with the empirical covariance matrix of the samples. Therefore, the empirical chain rule immediately follows from the chain rule of the aforesaid linear model. We will crucially use this result for structure learning in the realizable setting (\Cref{theo:treelearningrealizableintro}).

\paragraph{Lower bound of mutual information tester}
Now we prove that our result in \Cref{lem:miubintro} is tight, in the sense that given two zero-mean Gaussian random variables $X$ and $Y$, in order to distinguish between the cases if $I(X;Y)=0$ or $I(X;Y) \geq \eps$, $\Omega(1/\eps)$ samples are necessary. The proof follows from Fano's inequality, where we design a null hypothesis $H_0$ and an alternate hypothesis $H_1$, and we show that distinguishing between $H_0$ and $H_1$ requires $\Omega(1/\eps)$ samples. For the hard instances of lower bound, assume $ U_0, V_0, U_1, V_1 \sim N(0,1)$ be four iid random bits. The null hypothesis $H_0$ is defined as: (i) $X_0 \gets U_0$, (ii) $Y_0 \gets V_0$. The alternate hypothesis $H_1$ is defined as: (i) $X_1 \gets U_1$, (ii) $Y_1 \gets \sqrt{\eps} \cdot X_1 + V_1$.
Note that in contrast with the additive estimation of mutual information, this is quadratically easier. Interestingly, this lower bound construction also translates to a similar lower bound for the conditional mutual information tester and shows that the dependence on $\eps$ in the sample complexity of our conditional mutual information tester is also tight.

\subsection{Overview of structure learning result}

Now we will discuss our structure learning results. Let us start with the non-realizable setting.

\vspace{-8pt}

\paragraph{Non-realizable structure learning}
Here the distribution $P$ from which we are obtaining samples may not be a tree-structured distribution and our goal is to output an $\eps$-approximate tree $T$ of $P$. 
We use a black-box reduction from the work of \cite{bhattacharyya2023near}. They showed that if there exists an algorithm $\cA$ for estimating the mutual information between two random variables using $m(\eps,\delta)$ samples, then there exists an algorithm for learning the $\eps$-approximate tree $T$ that takes $m(\eps/n, \delta/n^2)$ samples from $P$. We also follow the same route. We will use our additive estimator of the mutual information between any two zero-mean Gaussian random variables from \Cref{cl:miestboundintro}, and obtain an algorithm for learning the $\eps$-approximate tree of $P$. This requires $\Oh(n^2/\eps^2 + \log n/\delta)$ samples from $P$.
We would like to point out that the authors in \cite{bhattacharyya2023near} studied the problem for the discrete distribution setting.
As a result, the additive mutual information tester that we use in our setting is different from that of \cite{bhattacharyya2023near}.

\vspace{-8pt}

\paragraph{Realizable structure learning}
Next, we consider the realizable setting, when the unknown distribution $P$ is promised to be a tree-structured distribution. Similar to non-realizable case, we follow the same route of 
\cite{bhattacharyya2023near}. There it was shown that if there exists an algorithm $\cA$ that takes $m(\eps,\delta)$ samples to distinguish between the cases if $I(X;Y \mid Z)=0$ from $I(X;Y \mid Z)\geq \eps$, then there exists an algorithm that takes $m(\eps/n, \delta/n^3)$ samples from $P$ and outputs an $\eps$-approximate tree $T$ of $P$. We will use our result on conditional mutual information tester from \Cref{lem:cmiubintro} that requires $\Oh(1/\eps \log 1/\delta)$ samples. This gives us an algorithm that takes $\Oh(n/\eps \log n/\delta)$ samples from $P$ in total. It is interesting to note that the sample complexity of structure learning in the realizable case is quadratically better as compared to the non-realizable case. Moreover, in \cite{bhattacharyya2023near}, the authors crucially used the chain rule of mutual information for their correctness analysis. However, as we will be estimating the conditional mutual information using regression, we crucially use the chain rule of empirical mutual information for correctness (\Cref{lem:cmichainruleempintro}).

\vspace{-11.35pt}

\paragraph{Lower bounds of structure learning results}

We finally show that our structure learning results for both non-realizable and realizable settings are tight with respect to $n$ and $\eps$, ignoring poly-logarithmic factors. In both these lower bound proofs, we construct a pair of hard instances, and we finally apply Fano's inequality along with coding theoretic arguments to show the desired lower bounds. For both these lower bounds, we will first prove the lower bound for $n=3$, and then generalize the lower bound for $n=3 \ell$ for some positive integer $\ell$. 

We prove that $\Omega(n^2/\eps^2)$ samples are necessary for outputting the $\eps$-approximate tree in the non-realizable case. For the hard instances of our lower bound in this setting for $n=3$ nodes, consider the following: let us assume $B_1, B_2, U_1, U_2, V_1, V_2, W_1,  W_2 \sim N(0,1)$ be eight iid random bits. The first tree $R_1$ is defined as follows:
\begin{align*}
 X_1 & \gets \paren{1+ \frac{\eps}{n}} \cdot B_1 + U_1 \\
 Y_1 & \gets \paren{1+ \frac{2\eps}{n}} \cdot B_1 + V_1 \\
 Z_1  & \gets \paren{1+ \frac{3 \eps}{n}} \cdot B_1 + W_1
\end{align*}

The second tree $R_2$ is defined as: 
\begin{align*}
X_2 & \gets \paren{1+ \frac{\eps}{n}} \cdot B_2 + U_2 \\
Y_2 & \gets \paren{1+ \frac{3 \eps}{n}} \cdot B_2 + V_2 \\
Z_2 & \gets \paren{1+ \frac{2\eps}{n}} \cdot B_2 + W_2
\end{align*}


In contrast to the non-realizable case, we show that $\Omega(n/\eps)$ samples are necessary for realizable setting. For $n=3$ nodes, consider the hard instances: let $B_1, B_2, U_1,  U_2, V_1, V_2, W_1, W_2 \sim N(0,1)$ be eight iid bits and $\gamma=1/2$. The first tree $R_1$ is defined as follows:
\begin{align*}
    Y_1  & \gets U_1\\
    Z_1  & \gets \paren{1-\gamma}Y_1 + W_1 \\
    X_1 & \gets \sqrt{\frac{\eps}{n}} Z_1 + V_1
\end{align*}


The second tree $R_2$ is defined as follows:
\begin{align*}
Y_2 & \gets U_2 \\
    Z_2 & \gets \paren{1-\gamma}Y_2 + W_2\\
    X_2 & \gets \sqrt{\frac{\eps}{n}} Y_2 + V_2\\    
\end{align*}
\paragraph{Extending the results to non-zero-mean Gaussian setting}
The testing and structure-learning results discussed above can be extended to the case when the Gaussian has non-zero mean. We consider the difference between two independent samples to reduce the problem to the zero-mean case. The detailed analysis can be found in \Cref{sec:nonzeromean} and \Cref{sec:structurelearnnonzeromean}, respectively.

\section{Proofs of mutual information testers}\label{sec:mitesterproof}
In this section, we will present all the proofs of our results on the additive estimator of mutual information and (conditional) mutual information testers.

\subsection{Transfer theorem for the non-zero mean case}
\label{sec:nonzeromean}

We start by presenting a transfer theorem that gives mutual information estimation/testing results for the non-zero-mean setting assuming such results for the zero-mean setting.

\begin{lem}\label{lem:diffsample}
    Let $(X_1,Y_1,Z_1)$ and $(X_2,Y_2,Z_2)$ be two independent and identically distributed 3-variate Gaussians. Let $(X',Y',Z')=(X_1-X_2,Y_1-Y_2,Z_1-Z_2)$. Let $M$ and $M'$ be the covariance matrices for $(X_1,Y_1,Z_1)$ and $(X',Y',Z')$ respectively.
    Then, the following holds:
    \begin{enumerate}
        \item $M'=2\cdot M$,
        \item $I(X_1;Z_1)=I(X';Z')$,
        \item $I(X_1;Y_1Z_1)=I(X';Y'Z')$,
        \item $I(X_1;Y_1 \mid Z_1)=I(X';Z' \mid Z')$.
    \end{enumerate}
\end{lem}

\begin{proof}
Note that $\Var(X')=2\Var(X_1)$, $\Cov(X',Y')=2\Cov(X_1,Y_1)$ and so on.

 Therefore, the covariance matrix for $(X',Y',Z')$ is clearly $2\cdot M$. This gives us the desired equalities from Fact~\ref{fact:gaussianMI}.
\end{proof}

Now we present our transfer results from zero to non-zero mean.

\begin{lem}\label{lem:transferMI}
    Let $(X,Y,Z)$ be jointly distributed as a Gaussian. Suppose, there is a tester that takes $m$ samples and satisfies the testing guarantees of either \Cref{cl:miestboundintro},  \Cref{lem:miubintro} or \Cref{lem:cmiubintro}, whenever $(X,Y,Z)$ are all zero-mean. Then, there is another tester that takes $2m$ samples and satisfies the guarantees of the same result where $(X,Y,Z)$ are not necessarily zero-mean.
\end{lem}

\begin{proof}
    We take the difference of two samples from a non-zero mean Gaussian. This transformation gives us a zero-mean sample while preserving the original mutual informations from \Cref{lem:diffsample}.
\end{proof}

\subsection{Additive estimation of mutual information} \label{sec:addtivemiest}

In this subsection, we give an algorithm for additive estimation of mutual information. Our main result in this section is the following.

\additiveUB*

Before proving the above lemma, we need some definitions and results discussed below.






Let $\widehat{\Sigma}_{X,Y}$ be the empirical covariance between the variables $X$ and $Y$ and $\widehat{M}_{X, Y}$ be a $2 \times 2$ matrix considering columns and rows of $X$ and $Y$ variables then,
$$
\widehat{M}_{X,Y} = \begin{bmatrix}
\widehat\Sigma_{X,X} & \widehat\Sigma_{X,Y} \\
\widehat\Sigma_{Y,X} & \widehat\Sigma_{Y,Y} \\
\end{bmatrix}
$$

Let $\widehat{\rho}_{X,Y}$ be the empirical correlation coefficient between the variables $X$ to $Y$. From \Cref{alg:condmitester}: 

\begin{equation}\label{eqn:rhoest}
\widehat{\rho}_{X,Y} = \frac{\widehat{\Sigma}_{X,Y}}{\sqrt{\widehat{\Sigma}_{X,X}\widehat{\Sigma}_{Y,Y}}}    
\end{equation}
$$ \widehat{I}(X;Y) = -\frac{1}{2}\log(1-\widehat{\rho}_{X,Y}^2)$$



Using the expression of $\widehat{\rho}_{X,Y}$ from \Cref{eqn:rhoest}, we have
\begin{equation}\label{eqn:estmutualinfxy}
\widehat{I}(X;Y) = -\frac{1}{2}\log{\left(\frac{\det\paren{{\widehat{M}_{X,Y}}}}{\widehat{\Sigma}_{X,X}\widehat{\Sigma}_{Y,Y}}\right)}    
\end{equation}

Now we have the following corollary.

\begin{coro}\label{coro:detdiffbound}
Let $(X,Y)$ be a zero-mean two-variate Gaussian with the covariance matrix $M_{X,Y}$. Let $\Sigma_{X,X}$ and $\Sigma_{Y,Y}$ denote the variances of $X$ and $Y$ respectively. With probability at least $1-3 \delta$, if $m>\Oh(1/\eps^2+ \log 1/\delta)$, the following holds:
$$-2\eps\leq \frac{1}{2}\log{\paren{\frac{\det\paren{{M_{X,Y}}}}{\Sigma_{X,X}\Sigma_{Y,Y}}}}- \frac{1}{2}\log{\paren{\frac{\det\paren{{\widehat{M}_{X,Y}}}}{\widehat{\Sigma}_{X,X}\widehat{\Sigma}_{Y,Y}}}} \leq 2\eps$$   
\end{coro}

\begin{proof}
First note that $\widehat{M}_{X,Y}$ is a positive semi-definite matrix. Following \Cref{obs:empcovmatrixbound}, we know that the following holds with probability at least $1-\delta$:
$$(1-\eps)M_{X,Y} \preceq \widehat{M}_{X,Y} \preceq (1+\eps)M_{X,Y}$$

Since $\widehat{M}_{X,Y}$ is a PSD matrix, 
we know that with probability at least $1-3 \delta$, the following three equations hold:
\begin{equation*}
(1-\eps)^2\det\paren{{M_{X,Y}}}\leq \det\paren{{\widehat{M}_{X,Y}}} \leq (1+\eps)^2\det{\paren{M_{X,Y}}}    
\end{equation*}

\begin{equation*}
(1-\eps)\Sigma_{X,X}\leq \widehat{\Sigma}_{X,X} \leq (1+\eps)\Sigma_{X,X}    
\end{equation*}

\begin{equation*}
(1-\eps)\Sigma_{Y,Y}\leq \widehat{\Sigma}_{Y,Y} \leq (1+\eps)\Sigma_{Y,Y}    
\end{equation*}

Using the expression of $\widehat{I}(X;Y)$ in \Cref{eqn:estmutualinfxy}, along with the above equations, we can say that, with probability $1-3 \delta$, the following holds:
$$-2\eps\leq \frac{1}{2}\log{\paren{\frac{\det\paren{{M_{X,Y}}}}{\Sigma_{X,X}\Sigma_{Y,Y}}}}- \frac{1}{2}\log{\paren{\frac{\det\paren{{\widehat{M}_{X,Y}}}}{\widehat{\Sigma}_{X,X}\widehat{\Sigma}_{Y,Y}}}} \leq 2\eps$$

This completes the proof of the corollary.
\end{proof}

Now we will prove that our additive mutual information estimator from \Cref{cl:miestboundintro} is tight, in the sense that, $\Omega(1/\eps^2)$ samples are necessary to estimate $I(X;Y)$ up-to an additive error of $\eps$. The formal statement of our result is as follows:

\begin{lem}[Additive MI estimation lower bound]\label{theo:miestlb}
Let $X, Y$ be two zero-mean Gaussian random variables. Given a parameter $\eps \in (0,1)$, in order to estimate $I(X;Y)$ up-to an additive error of $\eps$, with probability at least $9/10$, $\Omega(1/\eps^2)$ samples are necessary.    
\end{lem}

Before directly proceeding to the proof of the above lemma, we will first prove that a similar lower bound also holds for distinguishing between two linear models. Finally, we will prove our main result by the hardness of distinguishing these two models.

\begin{proof}

Let us assume $ U_0, V_0, U_1, V_1 \sim N(0,1)$ be four iid random bits. The null hypothesis $H_0$ is defined as follows:
\begin{align*}
 X_0 & \gets U_0 \\
 Y_0 & \gets \paren{\frac{1}{2}+ \eps} \cdot X_0 + V_0
\end{align*}

The alternate hypothesis $H_1$ is defined as follows:
\begin{align*}
 X_1 & \gets U_1 \\
 Y_1 & \gets \paren{\frac{1}{2}- \eps} \cdot X_1 + V_1
\end{align*}

The covariance matrix of $H_0$ is the following matrix:
$$\Sigma_0 = 
\begin{bmatrix}
1 & \frac{1}{2} + \eps  \\
\frac{1}{2} + \eps & \frac{5}{4} + \eps + \eps^2 
\end{bmatrix}
$$

Similarly, the covariance matrix of $H_1$ is the following:
$$\Sigma_1 = 
\begin{bmatrix}
1 & \frac{1}{2} - \eps  \\
\frac{1}{2} - \eps & \frac{5}{4} - \eps + \eps^2 
\end{bmatrix}
$$

Thus, $\det(\Sigma_0)=\det(\Sigma_1)=1$.
From \Cref{obs:klhellingerbounds}, we know that the $\kl$ between $H_0$ and $H_1$ can be expressed as follows:    
\begin{equation}\label{eqn:klmiestlb}
\kl(H_0 || H_1) = \frac{1}{2} \left(\tr(\Sigma_1^{-1} \Sigma_0) - 2 + \ln\paren{\frac{\det(\Sigma_1)}{\det(\Sigma_0)}}\right)    
\end{equation}

Putting the values of $\Sigma_0$ and $\Sigma_1$ in \Cref{eqn:klmiestlb}, we can say that $\kl(H_0||H_1) = 2 \eps^2$. In a similar manner, we also have $\kl(H_1||H_0) =  2 \eps^2$ as well. So, we have the following:
\begin{eqnarray*}
I(X_0; Y_0) &=& \frac{1}{2} \log\left(\frac{\sigma_{X_0}^2\sigma_{Y_0}^2}{\det (\Sigma_0)}\right)\\ &=& \frac{1}{2} \log \left(\frac{\frac{5}{4}+ \eps + \eps^2 }{\frac{1}{4}+ \eps + \eps^2}\right) \\
\end{eqnarray*}

Similarly, we also have
\begin{eqnarray*}
I(X_1; Y_1) &=& \frac{1}{2} \log\left(\frac{\sigma_{X_1}^2\sigma_{Y_1}^2}{\det (\Sigma_1)}\right)\\ &=& \frac{1}{2} \log \left(\frac{\frac{5}{4}- \eps + \eps^2 }{\frac{1}{4}- \eps + \eps^2}\right)
\end{eqnarray*}

So, 
\begin{eqnarray*}
I(X_1;Y_1) - I(X_0;Y_0) &=& \frac{1}{2}\log \left(
\frac{1-\frac{4}{5}\eps+\frac{4}{5}\eps^2}{1+\frac{4}{5}\eps+\frac{4}{5}\eps^2}
\right) \\ && \hspace{20 pt} + \frac{1}{2}\log
\left(
\frac{1+4\eps+4\eps^2}{1-4\eps+4\eps^2}
\right) \\ &=& \Theta(\eps).
\end{eqnarray*}

Consider any algorithm $\mathcal{A}$ that tries to distinguish $H_0$ from $H_1$ using $m$ samples. Let $D_0$ and $D_1$ be the distributions of $m$ samples from $H_0$ and $H_1$, respectively. Then, from \Cref{lem:pinkser}, we can say that $\mathsf{d_{TV}}(D_0,D_1)\le \sqrt{0.5\kl(D_0||D_1)}\le \sqrt{m\eps^2}$. Therefore, if $m=o(\eps^{-2})$, $\mathsf{d_{TV}}(D_0,D_1)=o(1)$ and hence $\mathcal{A}$ must err with at least $\frac{1}{2}-o(1)$ probability following \Cref{lem:lecam}.


Now we will prove our main lower bound result by contradiction. Consider an algorithm $\cA$ that can additively estimate $I(X;Y)$ with an error of $\Theta(\eps)$ using $o(1/\eps^2)$ samples.
From our construction above, since the gap between $I(X_0;Y_0)$ and $I(X_1;Y_1)$ is $\Theta(\eps)$, we could also use $\cA$ to distinguish between $H_0$ and $H_1$. However, this would contradict our lower bound from above. This implies that $\Omega(1/\eps^2)$ samples are necessary to additively estimate $I(X;Y)$ upto an error of $\eps$ with probability at least $9/10$. This completes the proof of our lower bound result.
\end{proof}

\begin{proof}[Proof of \Cref{cl:miestboundintro}]
We first assume the zero-mean case. From the \Cref{lem:transferMI} it will give us a result for the non-zero case with only twice the number of samples.

From \Cref{coro:detdiffbound}, we can say that with probability at least $1-\delta$, $\Oh(1/\eps^2 + \log 1/\delta)$ samples are sufficient to obtain $\widehat{I}(X;Y)$ such that $\size{\widehat{I}(X;Y)-I(X;Y)} \leq \eps$.

The tightness of $\eps$ on the sample complexity follows from \Cref{theo:miestlb}. 
\end{proof}

\subsection{Mutual Information testing upper bound}\label{sec:miub}

In this section, we design a tester for mutual information between two Gaussian random variables. In contrast to the previous section, it turns out that testing is quadratically easier compared to the estimation of mutual information. 
Our main result in this subsection is the following guarantee of the mutual information tester using $m=\Oh\paren{\eps^{-1}\log \delta^{-1}}$ samples.

\mitestub*

We show that $\Oh(\eps^{-1}\log \delta^{-1})$ samples suffice for~\Cref{alg:condmitester} to achieve the testing guarantee given in the above lemma.
We prove this by first showing that there exists an absolute constant $c$ for any $0<\lambda<1$  such that, whenever~\Cref{alg:condmitester} is run with $m\ge 8c^{-1}\lambda^{-2}\log \paren{6\delta^{-1}}$ samples, the estimated correlation will be close to the true correlation:
\begin{align}
    \rho_{X,Z}\paren{\frac{1}{\eta_X \eta_Z}} - \frac{\lambda}{\eta_X \eta_Z} \leq \wh{\rho}_{X,Z} \leq \rho_{X,Z}\paren{\frac{1}{\eta_X \eta_Z}} + \frac{\lambda}{\eta_X \eta_Z},\label{eqn:estcorr}
\end{align}
with probability at least $(1-\delta)$ for some $\eta_X,\eta_Z \in \sqrt{1\pm\lambda}$.

\begin{proof}
{Similar to before, we will prove this above result under the assumption that the Gaussian random variables are zero-mean. Then using  \Cref{lem:transferMI}, we will get a result for arbitrary mean Gaussian random variables as well with only twice the number of samples.}

    Without loss of generality, from \Cref{obs:mibasic}, let $(X,Z)$ follow the following equations for some $a,b,\alpha$.
    \begin{align*}
        Z &\sim N(0,a^2)\\
        V &\sim N(0, b^2)\\
        X &\longleftarrow \alpha Z + V
    \end{align*}

    We take $m$ samples from $(X,Z)$ :
    \[ 
    \begin{bmatrix} \bar{X} & \bar{Z} \end{bmatrix} = \begin{bmatrix}
    X_1 & Z_1 \\
    \vdots & \vdots \\
    X_m & Z_m
    \end{bmatrix}
    \]
    Our estimator will be $\wh{I}(X;Z) = \frac{1}{2}\log\frac{\paren{\bar{X}\cdot\bar{X}}\paren{\bar{Z}\cdot\bar{Z}}}{\paren{\bar{X}\cdot\bar{X}}\paren{\bar{Z}\cdot\bar{Z}} - \paren{\bar{X}\cdot\bar{Z}}^2}$.
    We will show that our estimated correlation coefficient $\wh{\rho}_{X,Z}:=\frac{\paren{\bar{X}\cdot\bar{Z}}}{\sqrt{
    \paren{\bar{X}\cdot\bar{X}}\paren{\bar{Z}\cdot\bar{Z}}}}$ is a good enough approximation for $\rho_{X,Z}:=\frac{\alpha a}{\sqrt{\alpha^2a^2 + b^2}}$.

    Note that from \Cref{fact:sqgauss}, $\frac{Z\cdot Z}{a^2},\frac{X\cdot X}{\alpha^2a^2 + b^2},\frac{\paren{X\cdot Z}}{a\sqrt{\alpha^2a^2 + b^2}}$ are all {sub-exponential random variables} each having sub-exponential norm at most $\frac{8}{3}$. Therefore from Bernstein's inequality (\Cref{lem:bernsteain}), we have the following:
    \begin{align*}
    \max &\left\{
    \mathrm{Pr}\left(\abs{\frac{1}{a^2} \paren{\bar{Z}\cdot\bar{Z}}-m} \geq t \right),
    \mathrm{Pr}\left(\abs{\frac{1}{\alpha^2a^2 + b^2}\paren{\bar{X}\cdot\bar{X}} -m}\geq t \right),    \mathrm{Pr}\left(\frac{\abs{\paren{\bar{X}\cdot\bar{Z}}-m\alpha a^2}}{a\sqrt{\alpha^2a^2 + b^2}} \geq t \right)
    \right\}\\
    &\leq 2\exp\left(-c \min\left(\frac{9t^2}{64m}, \frac{3t}{8}\right)\right)
    \end{align*}
    We set $t = m\lambda$ for some $0<\lambda<1$ to get the following:
    \begin{align*}
    \max &\left\{
    \mathrm{Pr}\left(\abs{\frac{1}{a^2} \paren{\bar{Z}\cdot\bar{Z}} - m} \geq m\lambda \right),
    \mathrm{Pr}\left(\abs{\frac{1}{\alpha^2a^2 + b^2}\paren{\bar{X}\cdot\bar{X}}-m} \geq m\lambda \right),    \mathrm{Pr} \left(\frac{\abs{\paren{\bar{X}\cdot\bar{Z}}-m\alpha a^2}}{a\sqrt{\alpha^2a^2 + b^2}} \geq m\lambda \right)
    \right\}\\
    &\leq 2\exp\left(-c \min\left(\frac{9m\lambda^2}{64}, \frac{3m\lambda}{8}\right)\right).
    \end{align*}
    Therefore, $\frac{1}{m}\paren{\bar{Z}\cdot\bar{Z}} \in a^2(1\pm\lambda)$, $\frac{1}{m}\paren{\bar{X}\cdot\bar{X}} \in (\alpha^2a^2+b^2)(1\pm\lambda)$ and  $\frac{1}{m}\paren{\bar{X}\cdot\bar{Z}} \in  \alpha a^2 \pm a\sqrt{\alpha^2a^2 + b^2}\lambda$ together with probability at least $(1-\delta)$ whenever $m \geq \max\braces{\frac{64 \log \frac{6}{\delta}}{9c\lambda^2}, \frac{8\log\frac{3}{\delta}}{3c\lambda}}$.
    Henceforth, we condition on this event. 
    
    Let $\eta_Z^2:=\frac{1}{m}\frac{\paren{\bar{Z}\cdot \bar{Z}}}{a^2}$ and $\eta_X^2 :=\frac{1}{m}\frac{\paren{\bar{Z}\cdot \bar{Z}}}{\paren{\alpha^2 a^2+b^2}}$.
    By previous discussion, $\eta_Z^2\in \paren{1\pm \lambda}$ and $\eta_X^2\in \paren{1\pm \lambda}$.    Hence, 
    \[
    \frac{\alpha a^2 - a\sqrt{\alpha^2a^2 + b^2}\lambda}{\sqrt{(\alpha^2a^2+b^2)} \eta_X \eta_Z a}
    \le \frac{\paren{\bar{X} \cdot \bar{Z}}}{\sqrt{\paren{\bar{X} \cdot \bar{X}}\paren{\bar{Z} \cdot \bar{Z}}}}
    \le \frac{\alpha a^2 + a\sqrt{\alpha^2a^2 + b^2}\lambda}{\sqrt{(\alpha^2a^2+b^2)}\eta_X \eta_Z a}.
    \]

    Or,
    \[
    \rho_{X,Z}\paren{\frac{1}{\eta_X \eta_Z}} - \frac{\lambda}{\eta_X \eta_Z} \leq \wh{\rho}_{X,Z} \leq \rho_{X,Z}\paren{\frac{1}{\eta_X \eta_Z}} + \frac{\lambda}{\eta_X \eta_Z} 
    \]

    We define $\wh{I}(X;Z):=\frac{1}{2}\log \paren{\frac{1}{1-\wh{\rho}_{X,Z}^2}}$ as our estimated mutual information. 
    Henceforth, we will fix $\lambda=\sqrt{\frac{\eps}{50}}$ to prove the lemma and assume $\eps<c$ for some suitable constant $c$. 
    
    Firstly, assume $I(X;Z)=0$. Then, $\rho_{X,Z}=0$ and hence  \[
    -\frac{\lambda}{1+\lambda} \leq {\wh{\rho}_{X,Z}} \leq \frac{\lambda}{1-\lambda}.
    \]
    Therefore, $\abs{\wh{\rho}_{X,Z}}\le \frac{\sqrt{\frac{\eps}{50}}}{1-\sqrt{\frac{\eps}{50}}} < 0.7$ and hence $\wh{I}(X;Z)< \wh{\rho}_{X,Z}^2 \le  \frac{\frac{\eps}{25}}{\paren{1-\sqrt{\frac{\eps}{50}}}^2} < \frac{\eps}{20}$.

    On the other hand, if $I(X;Z)\ge \eps$, then 
    $\rho_{X,Z}^2=1-e^{-2I(X;Z)} \ge 1-e^{-2\eps}>25\lambda^2$. Suppose, $\rho_{X,Z}>5\lambda$, then $\wh{\rho}_{X,Z}\ge {\frac{\rho_{X,Z}-\lambda}{\eta_X \eta_Z}} \ge {\frac{\rho_{X,Z}-\lambda}{\paren{1+\lambda}}}$. Otherwise, if $\rho_{X,Z}<-5\lambda$, then $\wh{\rho}_{X,Z}\le {\frac{\rho_{X,Z}+\lambda}{\eta_X \eta_Z}} \le {\frac{\rho_{X,Z}+\lambda}{\paren{1+\lambda}}}$. Combining the two cases, we can say that $\abs{\wh{\rho}_{X,Z}}\ge \frac{4}{5(1+\lambda)}\abs{\rho_{X,Z}}\ge \frac{1}{3} \sqrt{\eps}$. Therefore, $\wh{I}(X;Z)\ge \frac{1}{2}\log \frac{1}{1-\eps/3}\ge \frac{\eps}{6}\ge \frac{\eps}{8}.$ The proof of the time complexity of $\Oh(m)$ is direct and skipped.
\end{proof}

\subsection{Conditional mutual information testing upper bound}\label{sec:cmiub}
Next, we present our analysis of the conditional mutual information tester in \Cref{alg:condmitester}. Firstly, from \Cref{obs:zeromeangaussianstructure}, note that the joint distribution $(X,Y,Z)$ can be seen as following the linear system of equations given below. 
   \begin{align*} 
        Z & \sim N(0,a^2) \\
        V & \sim N(0,b^2) \\
        U & \sim N(0,c^2) \\  
        X & \longleftarrow \alpha Z + V\\
        Y & \longleftarrow \beta X + \gamma Z + U \\
    \end{align*}
We have $I(X;Y\mid Z)=I(X;Y\mid Z=0)$ from \Cref{obs:mibasic}. Note that conditioned on $Z=0$, the above system reduces to the following simpler system:
   \begin{align*} 
        V & \sim N(0,b^2)\\
        U & \sim N(0,c^2) \\  
        X & \longleftarrow V\\
        Y & \longleftarrow \beta V + U \\
    \end{align*}
The latter system has the same behavior as that of $(X-\alpha Z)$ and $(Y-(\alpha\beta+\gamma)Z)$. In particular, we get:\[
I(X;Y\mid Z) = I\left((X-\alpha Z,Y-(\alpha\beta+\gamma)Z\right).
\]

Therefore, if somebody gives us $\alpha,\beta,\gamma$, we can test $I(X; Y\mid Z)$ very easily by running the tester from \Cref{sec:miub} on the two transformed random variables: $(X-\alpha Z)$ and $(Y-(\alpha\beta+\gamma)Z)$ i.e. by performing ordinary least squares on the dataset $(\bar{X}-\alpha \bar{Z})$ and $(\bar{Y}-(\alpha\beta+\gamma)\bar{Z})$. However, we do not have the true values of $\alpha,\beta,\gamma$. Therefore, we take a two-step approach. First we estimate $\alpha,\beta,\gamma$ using the ordinary least squares on the dataset $(\bar{X},\bar{Y},\bar{Z})$ to get $\wh{\alpha},\wh{\beta},\wh{\gamma}$. Then we perform another ordinary least squares on the transformed dataset $(\bar{X}-\wh{\alpha} \bar{Z})$ and $(\bar{Y}-(\wh{\alpha}\wh{\beta}+\wh{\gamma})\bar{Z})$ using the regression coefficients obtained in the first step. This is exactly what \Cref{alg:condmitester} performs.

We start by showing the following fact.

\begin{fact}\label{fact:alphabetagammahat}
    Let $\wh{\alpha},\wh{\beta},\wh{\gamma}$ be the estimated regression coefficients between the random variable pairs $(X,Z),(Y,X),(Y,Z)$, respectively in \Cref{alg:condmitester}. Then the following equality holds:
    \begin{align*}        \paren{\wh{\alpha}\wh{\beta}+\wh{\gamma}}=\frac{\paren{\bar{Y}\cdot\bar{Z}}}{\paren{\bar{Z}\cdot\bar{Z}}}.
    \end{align*}
\end{fact}

\begin{proof}
Note that \Cref{alg:condmitester} uses simple linear regression to estimate the aforesaid correlation coefficients. Therefore,
\begin{align*}
        \wh{\alpha} &= \frac{\paren{\bar{Z} \cdot \bar{X}}}{\paren{\bar{Z} \cdot \bar{Z}}}\\ 
        \begin{bmatrix}
            \wh{\gamma} \\
            \wh{\beta}
        \end{bmatrix}
        &= {\begin{bmatrix}
        \bar{Z} \cdot \bar{Z} & \bar{X} \cdot \bar{Z} \\\bar{X} \cdot \bar{Z} & \bar{X} \cdot \bar{X}
        \end{bmatrix}}^{-1}    
        \begin{bmatrix}
            \bar{Z} \cdot \bar{Y} \\
            \bar{X} \cdot \bar{Z}
        \end{bmatrix}\\
        &= {\begin{bmatrix}
        \frac{\paren{\bar{X} \cdot \bar{X}}\paren{\bar{Z} \cdot \bar{Y}} -   \paren{\bar{X} \cdot \bar{Z}}^2}{\paren{\bar{Z} \cdot \bar{Z}}\paren{\bar{X} \cdot \bar{X}} - \paren{\bar{X} \cdot \bar{Z}}^2}\\
        \frac{-\paren{\bar{X} \cdot \bar{Z}}\paren{\bar{Z} \cdot \bar{Y}} + \paren{\bar{X} \cdot \bar{Z}}\paren{\bar{Z} \cdot \bar{Z}}}{\paren{\bar{Z} \cdot \bar{Z}}\paren{\bar{X} \cdot \bar{X}} - \paren{\bar{X} \cdot \bar{Z}}^2}
        \end{bmatrix}}\\
    \end{align*}

Therefore,
        \begin{align*}
        \paren{\wh{\alpha}\wh{\beta}+\wh{\gamma}} &= \frac{\frac{\paren{\bar{Z}\cdot\bar{X}}}{\paren{\bar{Z}\cdot\bar{Z}}}\paren{-\paren{\bar{X}\cdot\bar{Z}}\paren{\bar{Z}\cdot\bar{Y}}+\paren{\bar{X}\cdot\bar{Z}}\paren{\bar{Z}\cdot\bar{Z}}}+\paren{\bar{X}\cdot\bar{X}}\paren{\bar{Z}\cdot\bar{Y}}-\paren{\bar{X}\cdot\bar{Z}}^2}{\paren{\paren{\bar{X}\cdot\bar{X}}\paren{\bar{Z}\cdot\bar{Z}}-\paren{\bar{X}\cdot\bar{Z}}^2}}\nonumber\\
        &= \frac{
        -\paren{\bar{X}\cdot\bar{Z}}^2\paren{\bar{Z}\cdot\bar{Y}}/\paren{\bar{Z}\cdot\bar{Z}}+\paren{\bar{X}\cdot\bar{X}}\paren{\bar{Z}\cdot\bar{Y}}
        }{\paren{\bar{X}\cdot\bar{X}}\paren{\bar{Z}\cdot\bar{Z}}-\paren{\bar{X}\cdot\bar{Z}}^2}\nonumber\\
        &= \frac{
        \paren{\bar{Z}\cdot\bar{Y}}}{\paren{\bar{Z}\cdot\bar{Z}}}\label{eqn:alphabetagammhat}
        \end{align*}\qedhere
\end{proof}

Now we are ready to prove our result on conditional mutual information tester.

\cmitestub*
        
        

\begin{proof}
As before, we will prove this above result under the assumption that the Gaussian random variables are zero-mean. Then using \Cref{lem:transferMI}, we will able to extend this result to arbitrary-mean Gaussian random variable at the cost of doubling the sample complexity.

    Without loss of generality, we assume $\Pr[X,Y,Z]$ factorizes as $\Pr[Z]\Pr[X \mid Z]\Pr[Y \mid X,Z]$. In that case, following \Cref{obs:zeromeangaussianstructure}, we can say that $(X,Y,Z)$ follows the following equations :
    \begin{align*} 
        Z & \sim N(0,a^2) \\
        V & \sim N(0,b^2) \\
        U & \sim N(0,c^2) \\  
        X & \longleftarrow \alpha Z + V\\
        Y & \longleftarrow \beta X + \gamma Z + U \\
    \end{align*}

    We take $m$ samples from $(X,Y,Z)$ such that \[
    \begin{bmatrix} \bar{Z} & \bar{V} & \bar{U} \end{bmatrix} = \begin{bmatrix}
    Z_1 & V_1 & U_1 \\
    \vdots & \vdots & \vdots \\
    Z_m & V_m & U_m
    \end{bmatrix}\]
    \[
    \begin{bmatrix} \bar{Z} & \bar{X} & \bar{Y} \end{bmatrix} = 
    \begin{bmatrix}
    Z_1 & X_1 & Y_1 \\
    \vdots & \vdots & \vdots \\
    Z_m & X_m & Y_m
    \end{bmatrix} = 
    \begin{bmatrix}
    1 & 0 & 0 \\
    \alpha & 1 & 0\\
    \beta & \gamma & 1
    \end{bmatrix}
    \begin{bmatrix} \bar{Z} & \bar{V} & \bar{U} \end{bmatrix}^\top
    \]
    Then our estimator will be as follows. 
    We first estimate the regression coefficients between $(X,Z),(Y,X)$ and $(Y,Z)$ as $\wh{\alpha},\wh{\beta}$ and $\wh{\gamma}$, respectively using the empirical estimator. Finally, we estimate the regression coefficient between $\paren{X \mid Z = 0, Y \mid Z=0}$ as that of the pair $\paren{X - \wh{\alpha}Z, Y - \paren{\wh{\alpha}\wh{\beta} + \wh{\gamma}}Z}$ which we again estimate using the empirical estimator.

    Since \Cref{alg:condmitester} uses linear regression, we have the following using \Cref{fact:alphabetagammahat}:
    \begin{align*}
        \wh{\alpha} &= \frac{\paren{\bar{X}\cdot\bar{Z}}}{\paren{\bar{Z}\cdot\bar{Z}}}\\ &= \alpha + \frac{\paren{\bar{V}\cdot\bar{Z}}}{\paren{\bar{Z}\cdot\bar{Z}}}\\
        \paren{\wh{\alpha}\wh{\beta} + \wh{\gamma}} &= \frac{\paren{\bar{Z}\cdot\bar{Y}}}{\paren{\bar{Z}\cdot\bar{Z}}} \\&= \paren{\alpha\beta + \gamma} + \frac{\paren{\paren{\beta\bar{V} + \bar{U}}\cdot\bar{Z}}}{\paren{\bar{Z}\cdot\bar{Z}}}.
    \end{align*}
    From Bernstein's inequality (\Cref{lem:bernsteain}), we can say that whenever $m \geq \max\paren{\frac{64 \log \frac{6}{\delta}}{9c\lambda^2}, \frac{8\log\frac{3}{\delta}}{3c\lambda}}$, the following three statements hold together with probability at least $(1-\delta)$:
    \begin{align*}
        \paren{\bar{Z}\cdot\bar{Z}} &\in ma^2\paren{1 \pm \lambda}\\
        \paren{\bar{V}\cdot\bar{Z}} &\in \pm mab{ \lambda}\\
        \paren{\paren{\beta \bar{V} + \bar{U}}\cdot\bar{Z}} &\in \pm ma\lambda\sqrt{\beta^2b^2 + c^2}
    \end{align*}

    Therefore,
    \begin{align*}
        X - \wh{\alpha}Z &= \paren{\alpha - \wh{\alpha}}Z + V 
        \\&= V - \frac{\paren{\bar{V}\cdot\bar{Z}}}{\paren{\bar{Z}\cdot\bar{Z}}}Z\\
        &=V+\eta_{\alpha} Z\\
    \end{align*}
    where we have:
    \begin{align*}
      -\frac{b}{a}\paren{\frac{\lambda}{1 - \lambda}}  \le \eta_{\alpha}\le \frac{b}{a}\paren{\frac{\lambda}{1 - \lambda}}.
    \end{align*}
    Similarly, 
    \begin{align*}
        Y - \paren{\wh{\alpha}\wh{\beta} + \wh{\gamma}}Z &= \paren{\paren{\alpha\beta + \gamma} - \paren{\wh{\alpha}\wh{\beta} + \wh{\gamma}}}Z + \paren{\beta V + U} \\ 
        &= \paren{\beta V + U} - \frac{\paren{\paren{\beta\bar{V} + \bar{U}}\cdot\bar{Z}}}{\paren{\bar{Z}\cdot\bar{Z}}}Z\\
        &= \paren{\beta V + U} +\eta_{\alpha\beta+\gamma} Z\\
    \end{align*}
    where,
    \begin{align*}
        -\frac{\sqrt{\beta^2b^2 + c^2}}{a}\paren{\frac{\lambda}{1  - \lambda}} \le \eta_{\alpha\beta+\gamma} &\le  \frac{\sqrt{\beta^2b^2 + c^2}}{a}\paren{\frac{\lambda}{1  - \lambda}}.
    \end{align*}

    Therefore, 
    \begin{equation*}
        \begin{aligned}
        \rho_{\paren{X-\wh{\alpha}Z},\paren{Y-\paren{\wh{\alpha}\wh{\beta}+\wh{\gamma}}Z}} & = 
        \frac{
        \Cov\paren{
        \paren{X-\wh{\alpha}Z},
        \paren{Y-\paren{\wh{\alpha}\wh{\beta}+\wh{\gamma}}}Z
        }
        }{
        \sqrt{
        \Var\paren{X-\wh{\alpha}Z}
        \Var\paren{Y-\paren{\wh{\alpha}\wh{\beta}+\wh{\gamma}}}Z
        }
        }\\
        &= \frac{\Cov\paren{{\paren{V+\eta_{\alpha} Z}, \paren{\beta V + U} +\eta_{\alpha\beta+\gamma} Z}}}{\sqrt{\Var{\paren{V+\eta_{\alpha} Z}}\Var\paren{\paren{\beta V + U} +\eta_{\alpha\beta+\gamma} Z}}}\\
        &= \frac{\Cov\paren{\paren{\beta V + U},V} +  \Cov\paren{V,\eta_{\alpha\beta + \gamma}Z} + \Cov\paren{\eta_{\alpha}Z, \beta V + U} + \Cov\paren{\eta_{\alpha}Z, \eta_{\alpha\beta + \gamma}Z}}{
        {
        {
        \sqrt{\Var(V)+\eta_{\alpha}^2\Var(Z)}\sqrt{\Var\paren{\beta V + U}+\eta_{\alpha\beta+\gamma}^2\Var(Z)}
        }}}\\
        &= \frac{\Cov\paren{\paren{\beta V + U},V} + \eta_{\alpha}\eta_{\alpha\beta + \gamma}\Var\paren{Z}}{
        {
        {
        \sqrt{\Var(V)}\sqrt{\Var(\beta V+U)}
        \sqrt{1+\eta_{\alpha}^2\frac{\Var(Z)}{\Var(V)}}\sqrt{1+\eta_{\alpha\beta+\gamma}^2\frac{\Var(Z)}{\Var\paren{\beta V + U}}}
        }}}\\
        &=\frac{1}{\sqrt{1+\eta_{\alpha}^2\frac{\Var(Z)}{\Var(V)}}}
        \frac{1}{\sqrt{1+\eta_{\alpha\beta+\gamma}^2\frac{\Var(Z)}{\Var\paren{\beta V + U}}}}
        \paren{
        \rho_{\paren{\beta V +U},V}+\frac{
        \eta_{\alpha}\eta_{\alpha\beta + \gamma}\Var\paren{Z}
        }{
        \sqrt{\Var(V)}\sqrt{\Var(\beta V+U)}
        }
        }
        \\
        &=\frac{1}{\sqrt{1+\eta_{\alpha}^2\frac{\Var(Z)}{\Var(V)}}}
        \frac{1}{\sqrt{1+\eta_{\alpha\beta+\gamma}^2\frac{\Var(Z)}{\Var\paren{\beta V + U}}}}
        \paren{
        \rho_{\paren{X\mid Z=0},\paren{Y\mid Z=0}}+\frac{
        \eta_{\alpha}\eta_{\alpha\beta + \gamma}\Var\paren{Z}
        }{
        \sqrt{\Var(V)}\sqrt{\Var(\beta V+U)}
        }
        }
    \end{aligned}
    \end{equation*}

    Note that \begin{align*}
    - \paren{\frac{\lambda}{1- \lambda}} \le \frac{\eta_{\alpha}\sqrt{\Var(Z)}}{\sqrt{\Var(V)}} \le \paren{\frac{\lambda}{1- \lambda}}\\
    - \paren{\frac{\lambda}{1- \lambda}} \le \frac{\eta_{\alpha\beta + \gamma}\sqrt{\Var(Z)}}{\sqrt{\Var(\beta V+U)}} \le \paren{\frac{\lambda}{1- \lambda}}.
    \end{align*}

    We subsequently consider (we may reuse previous samples) $m'\ge \Omega(\lambda^{-2}\log {\delta^{-1}})$ samples from $X' = X - \wh{\alpha}Z$ and $Y' = Y - \paren{\wh{\alpha}\wh{\beta} + \wh{\gamma}}Z$ to estimate $\wh{\rho}_{X',Y'}$ using the empirical estimator. From~\Cref{eqn:estcorr}, we can say that, with probability at least $(1-\delta)$, the following holds:
    \[
    \rho_{X',Y'}\paren{\frac{1}{\eta_{X'} \eta_{Y'}}} - \frac{\lambda'}{\eta_{X'} \eta_{Y'}} \leq \wh{\rho}_{X',Y'} \leq \rho_{X',Y'}\paren{\frac{1}{\eta_{X'} \eta_{Y'}}} + \frac{\lambda'}{\eta_{X'} \eta_{Y'}} 
    \]
    where $\eta_{X'},\eta_{Y'}\in [\sqrt{1-\lambda'},\sqrt{1+\lambda'}]$. \\
    
    We define $\wh{I}(X;Y \mid Z) := \frac{1}{2}\log\paren{\frac{1}{1-\wh{\rho}_{X',Y'}^2}}$ as our estimated conditional mutual information. Henceforth, we will fix $\lambda = \lambda' = \sqrt{\frac{\eps}{50}}$ and assume $\eps < c$ for some suitable constant $c$.

Now we consider the two cases below:

    {\em Part (i).} Firstly, assume the case $I(X;Y \mid Z) = 0$. Then, $\rho_{\paren{X \mid Z = 0},\paren{ Y \mid Z = 0}} = 0$ and hence $\abs{\rho_{X',Y'}} \leq \frac{\lambda^2}{\paren{1-\lambda}^2 } $. Therefore, 
    \begin{align*}
    -\frac{\lambda^2}{(1-\lambda)^2}\frac{1}{1-\lambda'}-\frac{\lambda'}{1-\lambda'} \le  \wh{\rho}_{X',Y'} \le  \frac{\lambda^2}{(1-\lambda)^2}\frac{1}{1-\lambda'} +\frac{\lambda'}{(1-\lambda')} \le 0.7
    \end{align*}
    and $\wh{I}(X;Y \mid Z) \le \wh{\rho}_{X',Y'}^2 \le \paren{
    \frac{\lambda^2}{(1-\lambda)^2}\frac{1}{1-\lambda'} +\frac{\lambda'}{(1-\lambda')}
    }^2 \le \frac{\eps}{20}$.

    {\em Part (ii).} Now, consider the case when $I(X;Y \mid Z) \ge \eps$. Then, $\rho_{\paren{X\mid Z=0}, \paren{Y\mid Z=0}}^2= 1-e^{-2I(X;Y \mid Z)} \ge 1-e^{-2\eps}>25\lambda^2$. Suppose, $\rho_{\paren{X\mid Z=0}, \paren{Y\mid Z=0}}> 5\lambda$, then 
    \[
    \rho_{X',Y'}\ge 
    \frac{1}{\paren{1+\paren{\frac{\lambda}{1-\lambda}}^2}}\paren{\rho_{\paren{X\mid Z=0}, \paren{Y\mid Z=0}} - \paren{\frac{\lambda}{1-\lambda}}^2}
    \] 
    and therefore
    \begin{align*}
        \wh{\rho}_{X',Y'} \ge  \frac{1}{\paren{1+\paren{\frac{\lambda}{1-\lambda}}^2}}\paren{\rho_{\paren{X\mid Z=0}, \paren{Y\mid Z=0}} - \paren{\frac{\lambda}{1-\lambda}}^2} \paren{\frac{1}{1+\lambda'}} - \paren{\frac{\lambda'}{1-\lambda'}}.
    \end{align*}
    Otherwise, $\rho_{\paren{X\mid Z=0}, \paren{Y\mid Z=0}}< -5\lambda$, then 
    \begin{align*}
        \rho_{X',Y'} < \frac{1}{\paren{1+\paren{\frac{\lambda}{1-\lambda}}^2}}\paren{
        \rho_{\paren{X\mid Z=0}, \paren{Y\mid Z=0}} + \paren{\frac{\lambda}{1-\lambda}}^2 
        }
    \end{align*}

    and therefore
    \begin{align*}
        \wh{\rho}_{X',Y'} <  \frac{1}{\paren{1+\paren{\frac{\lambda}{1-\lambda}}^2}}\paren{
        \rho_{\paren{X\mid Z=0}, \paren{Y\mid Z=0}} + \paren{\frac{\lambda}{1-\lambda}}^2 
        } \paren{\frac{1}{1+\lambda'}} + \paren{\frac{\lambda'}{1-\lambda'}}
    \end{align*}

    Combining the two cases and setting $\lambda=\lambda'=\sqrt{\frac{\eps}{50}}$, we have the following:
    \begin{align*}
    \abs{
    \wh{\rho}_{X',Y'}}&\ge \paren{
    \frac{1}{\paren{1+\paren{\frac{\lambda}{1-\lambda}}^2}}\paren{\abs{\rho_{\paren{X\mid Z=0}, \paren{Y\mid Z=0}}} - \paren{\frac{\lambda}{1-\lambda}}^2} \paren{\frac{1}{1+\lambda}} - \paren{\frac{\lambda}{1-\lambda}}
    }\\
    &\ge \paren{\frac{1}{1+\lambda}}\paren{\frac{1}{1+\paren{\frac{\lambda}{1-\lambda}}^2}}\abs{\rho_{\paren{X\mid Z=0}, \paren{Y\mid Z=0}}}-\paren{
    \frac{
    \paren{
    \frac{\lambda}{1-\lambda}
    }^2\paren{\frac{1}{1+\lambda}
    }}{{\paren{1+\paren{\frac{\lambda}{1-\lambda}}^2}
    }
    }
    +\paren{\frac{\lambda}{1-\lambda}}
    }\\
    &={\frac{\sqrt{\eps}}{4}},
    \end{align*}

    Therefore, $\wh{I}(X;Y\mid Z)\ge \frac{1}{2}\log \frac{1}{1-0.25\eps}\ge \eps/8$. 

    The total number of samples is $\max\braces{m,m'} = \Oh(1/\eps \log 1/\delta)$ and the total failure probability is $2\delta$. {We re-scale it to $\delta$, however, that would only change the sample complexity only by a factor of $2$. This completes the proof of the lemma.}  The proof of the time complexity of $\Oh(m)$ is direct and omitted.
\end{proof}

\subsection{Chain rule of empirical conditional mutual information}\label{sec:cmichainrule}

In this section, we prove that the chain rule of empirical mutual information holds in our setting.

\chainrulecmi*
    

\begin{proof}
    Let the matrix representation of $X$, $Y$, and $Z$ for $m$ samples be the following:
    \[
    \begin{bmatrix} \bar{X} & \bar{Y} & \bar{Z} \end{bmatrix} = \begin{bmatrix}
    X_1 & Y_1 & Z_1 \\
    \vdots & \vdots & \vdots \\
    X_m & Y_m & Z_m
    \end{bmatrix}
    \]

    Let $\wt{X} = \bar{X} - \wh{\alpha}\bar{Z}$
    and $\wt{Y} = \bar{Y} - \paren{\wh{\alpha}\wh{\beta} + \wh{\gamma}}\bar{Z}$,
    where, from \Cref{alg:condmitester} and \Cref{fact:alphabetagammahat}, we have:
    \begin{align*}
        \wh{\alpha} &= \frac{\paren{\bar{Z} \cdot \bar{X}}}{\paren{\bar{Z} \cdot \bar{Z}}} \\
        \paren{\wh{\alpha}\wh{\beta}+\wh{\gamma}}&=\frac{\paren{\bar{Z} \cdot \bar{Y}}}{\paren{\bar{Y} \cdot \bar{Y}}} 
    \end{align*}

    We now define the following matrices: 
    \begin{align*}
        \wh{M} &= \begin{bmatrix}
            \frac{1}{m}\paren{\bar{Z}\cdot\bar{Z}} & \frac{1}{m}\paren{\bar{Z}\cdot\bar{X}} &
            \frac{1}{m}\paren{\bar{Z}\cdot\bar{Y}} \\
            \frac{1}{m}\paren{\bar{X}\cdot\bar{Z}} & \frac{1}{m}\paren{\bar{X}\cdot\bar{X}} &
            \frac{1}{m}\paren{\bar{X}\cdot\bar{Y}} \\
            \frac{1}{m}\paren{\bar{Y}\cdot\bar{Z}} & \frac{1}{m}\paren{\bar{Y}\cdot\bar{X}} &
            \frac{1}{m}\paren{\bar{Y}\cdot\bar{Y}}
        \end{bmatrix}\\
        \wh{M}_{ZX} &= \begin{bmatrix}
             \frac{1}{m}\paren{\bar{Z}\cdot\bar{Z}} & \frac{1}{m}\paren{\bar{Z}\cdot\bar{X}}\\
              \frac{1}{m}\paren{\bar{X}\cdot\bar{Z}} & \frac{1}{m}\paren{\bar{X}\cdot\bar{X}}\\
        \end{bmatrix}\\
        \wt{M}_{XY} &= \begin{bmatrix}
            \frac{1}{m}\paren{\wt{X}\cdot\wt{X}} & \frac{1}{m}\paren{\wt{X}\cdot\wt{Y}}\\
            \frac{1}{m}\paren{\wt{X}\cdot\wt{Y}} &
            \frac{1}{m}\paren{\wt{Y}\cdot\wt{Y}}\\
        \end{bmatrix}
    \end{align*}
    
    Our algorithm gives the following estimates
    \begin{align*}
        \wh{I}(X;Z) = \frac{1}{2}\log\frac{\paren{\frac{1}{m}\paren{\bar{X}\cdot\bar{X}}}\paren{\frac{1}{m}\paren{\bar{Z}\cdot\bar{Z}}}}{\det\paren{{\wh{M}_{ZX}}}}\\
        \wh{I}(X;Y \mid Z) = \frac{1}{2}\log\frac{\paren{\frac{1}{m}\paren{\wt{X}\cdot\wt{X}}}{\paren{\frac{1}{m}\paren{\wt{Y}\cdot\wt{Y}}}}}
        {\det\paren{{\wt{M}_{XY}}}}
    \end{align*}

    We will show that \begin{align}
    \wh{I}(X;Z)+\wh{I}(X;Y \mid Z)=\wh{I}(X;YZ)\label{eqn:empchainruleproof}
    \end{align}
    where 
    \begin{align*}
        \wh{I}(X;YZ) = \frac{1}{2}\log\frac{\paren{\frac{1}{m}\paren{\bar{X}\cdot\bar{X}}}\det\paren{{\wh{M}_{ZX}}}}{\det\paren{{\wh{M}}}}.
    \end{align*}
    The other equality
    $
    \wh{I}(X;YZ)=\wh{I}(X;Z)+\wh{I}(X;Y\mid Z)
    $
    will follow analogously and is skipped.

    We will prove the displayed equation above by setting up a linear system whose covariance matrix coincides with $\wh{M}$. Let:
    \begin{align*} 
        Z' & \sim N(0,a'^2) \\
        V' & \sim N(0,b'^2) \\
        U' & \sim N(0,c'^2) \\  
        X' & \longleftarrow \alpha Z' + V'\\
        Y' & \longleftarrow \beta X' + \gamma Z' + U' 
    \end{align*}

It is easy to see that the following is the unique solution of the above linear system if its covariance matrix coincides with $\wh{M}$.
    \begin{align*}
        a'^2 &= \frac{1}{m}\paren{\bar{Z} \cdot \bar{Z}} \\
        \alpha' &= \frac{\paren{\bar{X}\cdot\bar{Z}}}{\paren{\bar{Z}\cdot\bar{Z}}} \\
        b'^2 &= \frac{1}{m}\paren{\frac{\paren{\bar{X}\cdot\bar{X}}\paren{\bar{Z}\cdot\bar{Z}} - \paren{\bar{X}\cdot\bar{Z}}^2}{\paren{\bar{Z}\cdot\bar{Z}}}}\\
        \beta' &= \frac{\paren{\bar{X}\cdot\bar{Y}}\paren{\bar{Z}\cdot\bar{Z}} - \paren{\bar{X}\cdot\bar{Z}}\paren{\bar{Y}\cdot\bar{Z}}}{\paren{\bar{X}\cdot\bar{X}}\paren{\bar{Z}\cdot\bar{Z}} - \paren{\bar{X}\cdot\bar{Z}}^2}\\
        \gamma' &= \frac{\paren{\bar{Y}\cdot\bar{Z}}\paren{\bar{X}\cdot\bar{X}} - \paren{\bar{X}\cdot\bar{Z}}\paren{\bar{X}\cdot\bar{Y}}}{\paren{\bar{X}\cdot\bar{X}}\paren{\bar{Z}\cdot\bar{Z}} - \paren{\bar{X}\cdot\bar{Z}}^2}\\
        \end{align*}
        \begin{align*}
        mc'^2 &= \frac{\paren{\bar{X}\cdot\bar{X}}\paren{\bar{Y}\cdot\bar{Y}}\paren{\bar{Z}\cdot\bar{Z}} - \paren{\bar{X}\cdot\bar{X}}\paren{\bar{Y}\cdot\bar{Z}}^2 - \paren{\bar{X}\cdot\bar{Z}}^2\paren{\bar{Y}\cdot\bar{Y}} - \paren{\bar{X}\cdot\bar{Y}}^2\paren{\bar{Z}\cdot\bar{Z}} + 2\paren{\bar{X}\cdot\bar{Y}}\paren{\bar{Y}\cdot\bar{Z}}\paren{\bar{X}\cdot\bar{Z}}}{\paren{\bar{X}\cdot\bar{X}}\paren{\bar{Z}\cdot\bar{Z}} - \paren{\bar{X}\cdot\bar{Z}}^2}\\
         &= \frac{\det\paren{
         \begin{bmatrix}
            \paren{\bar{X}\cdot\bar{X}} & \paren{\bar{X}\cdot\bar{Y}} &
            \paren{\bar{X}\cdot\bar{Z}} \\
            \paren{\bar{X}\cdot\bar{Y}} & \paren{\bar{Y}\cdot\bar{Y}} &
            \paren{\bar{Y}\cdot\bar{Z}} \\
            \paren{\bar{X}\cdot\bar{Z}} & \paren{\bar{Y}\cdot\bar{Z}} &
            \paren{\bar{Z}\cdot\bar{Z}}
        \end{bmatrix}}}{{\paren{\bar{X}\cdot\bar{X}}\paren{\bar{Z}\cdot\bar{Z}} - \paren{\bar{X}\cdot\bar{Z}}^2}}
    \end{align*}

    Note that $a'$, $b'$, $c'$ will always be well-defined whenever $\det\paren{\wh{M}_{ZX}}$ and $\bar{Z}\cdot\bar{Z}$ are both positive, which happens with probability 1 for multivariate Gaussians. Also, note that $I(X';Z') = \wh{I}(X;Z)$ and $I(X';Y'Z') = \wh{I}(X;YZ)$ by construction. It remains to show that $\wh{I}(X;Y \mid Z) = I(X';Y'\mid Z') = \frac{1}{2}\log\paren{\frac{\beta'^2 b'^2 + c'^2}{c'^2}}$; equivalently,
    $$\frac{\paren{\frac{1}{m}\paren{\wt{X}\cdot\wt{X}}}{\paren{\frac{1}{m}\paren{\wt{Y}\cdot\wt{Y}}}}}
        {\det\paren{{\wt{M}_{XY}}}} = \paren{\frac{\beta'^2 b'^2 + c'^2}{c'^2}}.$$

    We show this as follows. First of all,

        \begin{align*}
            \wt{X}&=\bar{X}-\paren{\frac{\bar{X}\cdot\bar{Z}}{\bar{Z}\cdot\bar{Z}}}\cdot\bar{Z}
        \end{align*}

        \begin{align*}
            \wt{Y}&=\bar{Y}-\paren{\frac{\bar{Y}\cdot\bar{Z}}{\bar{Z}\cdot\bar{Z}}}\cdot\bar{Z}
        \end{align*}
        
        \begin{align*}
            \wt{X}\cdot \wt{X} &= \bar{X}\cdot\bar{X}-\frac{\paren{\bar{X}\cdot\bar{Z}}^2}{\bar{Z}\cdot\bar{Z}}
        \end{align*}

        \begin{align*}
            \wt{Y}\cdot \wt{Y} = \bar{Y}\cdot\bar{Y}-\frac{\paren{\bar{Y}\cdot\bar{Z}}^2}{\bar{Z}\cdot\bar{Z}}  
        \end{align*}

        \begin{align*}
            \wt{X}\cdot\wt{Y} = \bar{X}\cdot\bar{Y}-\frac{\paren{\bar{X}\cdot\bar{Z}}\paren{\bar{Y}\cdot\bar{Z}}}{\paren{\bar{Z}\cdot\bar{Z}}}
        \end{align*}

        \begin{align*}        
        m^2\cdot\det\paren{\wt{M}_{{X}{Y}}} &= \paren{\bar{X}\cdot\bar{X}-\frac{\paren{\bar{X}\cdot\bar{Z}}^2}{(\bar{Z}\cdot\bar{Z})}}\paren{\bar{Y}\cdot\bar{Y}-\frac{(\bar{Y}\cdot\bar{Z})^2}{(\bar{Z}\cdot\bar{Z})}} - \paren{\bar{X}\cdot\bar{Y}-\frac{(\bar{X}\cdot\bar{Z})(\bar{Y}\cdot\bar{Z})}{(\bar{Z}\cdot\bar{Z})}}^2\\
        &= \paren{(\bar{X}\cdot\bar{X})(\bar{Y}\cdot\bar{Y})-(\bar{X}\cdot\bar{Y})^2}-\frac{(\bar{Y}\cdot\bar{Y})(\bar{X}\cdot\bar{Z})^2}{(\bar{Z}\cdot\bar{Z})}-\frac{(\bar{X}\cdot\bar{X})(\bar{Y}\cdot\bar{Z})}{(\bar{Z}\cdot\bar{Z})}+\frac{2(\bar{X}\cdot\bar{Y})(\bar{X}\cdot\bar{Z})(\bar{Y}\cdot\bar{Z})}{(\bar{Z}\cdot\bar{Z})}\\
        &= \frac{m}{\paren{\bar{Z}\cdot\bar{Z}}}\cdot c'^2\cdot\paren{(\bar{X}\cdot\bar{X})(\bar{Z}\cdot\bar{Z})-(\bar{X}\cdot\bar{Z})^2}            
        \end{align*}

        \begin{align*}
        \frac{\paren{\frac{1}{m}\paren{\wt{X}\cdot\wt{X}}}{\paren{\frac{1}{m}\paren{\wt{Y}\cdot\wt{Y}}}}}
        {\det\paren{{\wt{M}_{XY}}}}
        &=\frac{1}{m}\frac{
        \paren{\paren{\bar{X}\cdot\bar{X}}\paren{\bar{Z}\cdot\bar{Z}}-\paren{\bar{X}\cdot\bar{Z}}^2}
        \paren{\paren{\bar{Y}\cdot\bar{Y}}- \frac{\paren{\bar{Y}\cdot\bar{Z}}^2}{\paren{\bar{Z}\cdot\bar{Z}}}}
        }{
        c'^2\cdot\paren{(\bar{X}\cdot\bar{X})(\bar{Z}\cdot\bar{Z})-(\bar{X}\cdot\bar{Z})^2} 
        }\\
        &=\frac{1}{m}\frac{\paren{\bar{Y}\cdot\bar{Y}}\paren{\bar{Z}\cdot\bar{Z}}-\paren{\bar{Y}\cdot\bar{Z}}^2}{c'^2\paren{\bar{Z}\cdot\bar{Z}}}\\
        &=\frac{\beta'^2 b'^2 + c'^2}{c'^2}.
        \end{align*}

        Therefore, $\wh{I}(X;Y \mid Z) = I(X';Y' \mid Z')$. Now we will apply the chain rule for the above linear system to get 
        \[
        I(X';Y'Z') = I(X';Y') + I(X';Y' \mid Z').
        \]
        Replacing the corresponding $\wh{I}$ values gives us \Cref{eqn:empchainruleproof}.
\end{proof}

\subsection{Lower bound of mutual information tester}\label{sec:mitestlb}

In this subsection, we will prove that our mutual information tester from \Cref{sec:miub} is tight, in the sense that, $\Omega(1/\eps)$ samples are necessary to distinguish between the cases $I(X;Y)=0$ and $I(X;Y) \geq \eps$. The formal statement of our result is as follows:

\mitestlb*

Before directly proceeding to the proof of the above theorem, we will first prove that a similar lower bound also holds for distinguishing between the following two linear models. Finally, we will prove our main result by the hardness of distinguishing these two models.

\begin{proof}

Let us assume $ U_0, V_0, U_1, V_1 \sim N(0,1)$ be four iid random bits. The null hypothesis $H_0$ is defined as follows:
\begin{align*}
 X_0 & \gets U_0 \\
 Y_0 & \gets V_0
\end{align*}

The alternate hypothesis $H_1$ is defined as follows:
\begin{align*}
 X_1 & \gets U_1 \\
 Y_1 & \gets \sqrt{\eps} \cdot X_1 + V_1
\end{align*}

The covariance matrix of $H_0$ is the identity matrix:
$$\Sigma_0 = 
\begin{bmatrix}
1 & 0  \\
0 & 1 
\end{bmatrix}
$$

Similarly, the covariance matrix of $H_1$ is the following:
$$\Sigma_1 = 
\begin{bmatrix}
1 & \sqrt{\eps}  \\
\sqrt{\eps} & 1+ \eps 
\end{bmatrix}
$$

Thus, $\det(\Sigma_0)=\det(\Sigma_1)=1$.
From \Cref{obs:klhellingerbounds}, we know that the $\kl$ between $H_0$ and $H_1$ can be expressed as follows:    
\begin{equation}\label{eqn:klmilb}
\kl(H_0 || H_1) = \frac{1}{2} \left(\tr(\Sigma_1^{-1} \Sigma_0) - 2 + \ln\paren{\frac{\det(\Sigma_1)}{\det(\Sigma_0)}}\right)    
\end{equation}

Putting the values of $\Sigma_0$ and $\Sigma_1$ in \Cref{eqn:klmilb}, we can say that $\kl(H_0||H_1) = \eps/2$. In a similar manner, we can also show that $\kl(H_1||H_0) = \eps/2$. Since $I(X_1; Y_1) = \frac{1}{2} \log\left(\frac{\sigma_{X_1}^2\sigma_{Y_1}^2}{\det (\Sigma_1)}\right)$, we have the following:
\begin{eqnarray*}
I(X_1; Y_1) &=& \frac{1}{2} \log\left(\frac{\sigma_{X_1}^2\sigma_{Y_1}^2}{\det (\Sigma_1)}\right)\\ &=& \frac{1}{2} \log \left(\frac{1+ \eps}{1}\right) \\ &=& \Theta(\eps)  
\end{eqnarray*}
Moreover, note that $I(X_0;Y_0) = 0$.

Consider any algorithm $\mathcal{A}$ that tries to distinguish $H_0$ from $H_1$ using $m$ samples. Let $D_0$ and $D_1$ be the distributions of $m$ samples from $H_0$ and $H_1$, respectively. Then, from \Cref{lem:pinkser}, we can say that $\mathsf{d_{TV}}(D_0,D_1)\le \sqrt{0.5\kl(D_0||D_1)}\le \sqrt{m\eps^2}$. Therefore, if $m=o(\eps^{-1})$, $\mathsf{d_{TV}}(D_0,D_1)=o(1)$ and hence $\mathcal{A}$ must err with at least $\frac{1}{2}-o(1)$ probability following \Cref{lem:lecam}.

Now we will prove our main result by contradiction. Consider an algorithm $\cA$ that can distinguish between $I(X;Y)=0$ and $I(X;Y) \geq \eps$ using $o(1/\eps)$ samples.
From our construction above, since $I(X_0;Y_0) =0$ and $I(X_1;Y_1) = \Theta(\eps)$, we could also use $\cA$ to distinguish between $H_0$ and $H_1$. However, this would contradict our lower bound from above. This implies that $\Omega(1/\eps)$ samples are necessary to distinguish between $I(X;Y)=0$ and $I(X;Y) \geq \eps$ with probability at least $9/10$. This completes the proof of our theorem.    
\end{proof}

\section{Proofs of structure learning results}

In this section, we will present the proofs of our results on structure learning.

\subsection{Transfer theorem for the non-zero mean case}\label{sec:structurelearnnonzeromean}
In this section, give a transfer theorem that says that any approximate structure learning algorithm for the zero-mean case can be used to design an algorithm for the arbitrary mean case with only doubling the sample complexity.

\begin{theo}\label{theo:treelearningnonrealizablenonzeromean}
Let $P$ be an unknown $n$-variate arbitrary-mean Gaussian distribution. Suppose there exists a structure learning algorithm that takes $m$ samples and satisfies the guarantees of \Cref{theo:treelearningnonrealizableintro} whenever $P$ is an $n$-variate zero-mean Gaussian distribution. Then there exists another structure learning algorithm that takes $2m$ samples and satisfies the guarantees of \Cref{theo:treelearningnonrealizableintro} when $P$ is not necessarily zero-mean Gaussian distribution. 
\end{theo}

\begin{proof}
Let $P$ be a $N(\mu, \Sigma)$ distribution for unknown mean $\mu$ and unknown covariance matrix $\Sigma$.

We take two independent samples $X$ and $Y$ from $P$. We define a new random variable $Q:= X-Y = W$. Therefore, $Q$ follows $N(0, 2\Sigma)$ distribution from \Cref{lem:diffsample}. Thus, it is clear that for every pair of nodes $i \neq j \in [n]$, we have the following:
\begin{equation}\label{eqn:mipq}
I_P(X_i;X_j)= I_Q(W_i;W_j).    
\end{equation}
where $I_P(X_i;X_j)$ denotes the mutual information between $X_i$ and $X_j$ with respect to $P$. Similarly, we define $I_Q(W_i;W_j)$.

Let us assume that $T_P$ and $T_Q$ be the best trees with respect to $P$ and $Q$, respectively. 
Moreover, $\widehat{T}$ be the tree obtained by Chow-Liu algorithm (\Cref{theo:chowliualgo}). Then from the guarantee of the Chow-Liu algorithm, we can say the following:
\begin{equation}\label{eqn:chowliuguarantee}
\kl(P||P_{T}) = J_P - \mathsf{wt}_P(T)    
\end{equation}

For any spanning tree $\widehat{T}$, we have the following claim.
\begin{cl}\label{cl:diffpq}
$\kl(P||P_{\widehat{T}}) - \kl(P||P_{T_P})=\kl(Q||Q_{\widehat{T}}) - \kl(Q||Q_{T_Q})$    
\end{cl}

\begin{proof}
From \Cref{eqn:chowliuguarantee}, we can say that
$$\kl(P||P_{\widehat{T}}) - \kl(P||P_{T_P}) = \mathsf{wt}_P(T_P) - \mathsf{wt}_P(\widehat{T})$$

Similarly, we also have:
$$\kl(Q||Q_{\widehat{T}}) - \kl(Q||Q_{T_Q}) =
\mathsf{wt}_Q(T_Q) - \mathsf{wt}_Q(\widehat{T})$$

From \Cref{eqn:mipq}, we can say that $\mathsf{wt}_P(\widehat{T}) = \mathsf{wt}_Q(\widehat{T})$. In a similar fashion, we can also say that $\mathsf{wt}_P(T_P) = \mathsf{wt}_Q(T_Q)$. Combining the above, we have the result.    
\end{proof}

Since $Q \sim N(0, 2 \Sigma)$, we can apply our structure learning algorithm from \Cref{theo:treelearningnonrealizableintro} to obtain a tree $\wh{T}$. Moreover, as \Cref{cl:diffpq} holds, if $\wh{T}$ is $\eps$-approximate for $Q$, $\wh{T}$ will also be an $\eps$-approximate tree for $P$ as well. 
Finally, note that in order to simulate a single sample from $Q$, it is sufficient to obtain $2$ samples from $P$. Thus, if the structure learning algorithm for zero-mean Gaussian distribution requires $m$ samples to output the $\eps$-approximate tree, then the new structure learning algorithm uses $2m$ samples in total. This completes the proof of the theorem.
\end{proof}

Now we study the structure learning problem in the realizable setting. It turns out that similar argument also holds in this scenario. We have the following result.

\begin{theo}\label{theo:treelearningrealizablenonzeromean}
Let $P$ be an unknown $n$-variate arbitrary-mean tree-structured Gaussian distribution. Suppose there exists a structure learning algorithm that takes $m$ samples and satisfies the guarantees of \Cref{theo:treelearningrealizableintro} whenever $P$ is an $n$-variate zero-mean tree-structured Gaussian distribution. Then there exists another structure learning algorithm that takes $2m$ samples and satisfies the guarantees of \Cref{theo:treelearningrealizableintro} when $P$ is not necessarily zero-mean Gaussian distribution.   
\end{theo}

\begin{proof}
The proof follows in a similar fashion of \Cref{theo:treelearningnonrealizablenonzeromean} and is skipped.    
\end{proof}

\subsection{Learning Tree Structures: non-realizable case}\label{sec:nonrealizable}

Our main goal in this subsection is to prove the following approximate structure learning result when the underlying multivariate Gaussian distribution need not be tree-structured but it is known to be zero-mean. 
Theorem~\ref{theo:treelearningnonrealizableintro} gives a guarantee for the Chow-Liu algorithm that simply returns the maximum spanning tree of the pairwise estimated mutual informations. For the mutual information estimation, we use the empirical estimator from Section~\ref{sec:addtivemiest}.

\treelearningnonrealizable*

The authors in \cite{bhattacharyya2023near} have recently looked at the analogous structure learning problem for the discrete case where the unknown distribution $P$ is over $\Sigma^n$ for some finite alphabet $\Sigma$. They have given a black-box reduction from the non-realizable structure learning problem to the problem of additive estimation of mutual information. Fortunately for us, the same reduction extends to the case when $P$ is a zero-mean multivariate Gaussian distribution. Therefore, once we have an efficient algorithm for the additive estimation of mutual information, that we show in Section~\ref{sec:addtivemiest}, it directly gives us an algorithm for the non-realizable structure learning problem.


\begin{theo}[Restatement of Lemma $1.1$ of \cite{bhattacharyya2023near}]\label{lem:nonrealizabletheodiscrete}
Let $P$ be an unknown distribution defined over $\R^n$~\footnote{We would like to point out that the result of \cite{bhattacharyya2023near} was studied for the discrete setting. However, the same proof also extends for the case when the distribution is defined over $\R^n$. }. Consider an algorithm $\cA$ that takes $m(\eps,\delta)$ samples, outputs an estimate $\wh{I}(X;Y)$ of the mutual information between any two random variables $X$ and $Y$ such that with probability at least $(1-\delta)$, the following holds:
\[
|\wh{I}(X;Y)- I(X;Y)| \leq \eps.
\]

Then there exists an algorithm, that given access to $\cA$, can construct an $\eps$-approximate tree by taking $\Oh(m(\eps/n, \delta/n^2))$ samples from $P$.
\end{theo}

As discussed before, the above reduction combined with our additive estimator for Gaussian mutual information, directly gives us the guarantee for the structure learning problem for the non-realizable case.

\begin{proof}[Proof of \Cref{theo:treelearningnonrealizableintro}]
We first assume the zero-mean case. By \Cref{theo:treelearningnonrealizablenonzeromean}, we'll immediately get an algorithm for the arbitrary mean case with twice the number of samples.

Let us assume $P$ be the unknown zero-mean Gaussian distribution, and $T^*$ be the tree minimizing $\kl(P||P_{T^*})$. From \Cref{cl:miestboundintro}, we know that $\Oh(1/\eps^2 + \log 1/\delta)$ samples from $P$ are sufficient to estimate the mutual information between any two Gaussian random variables. Thus, we will apply \Cref{lem:nonrealizabletheodiscrete} to obtain an $\eps$-approximate tree. This requires $\Oh(\frac{n^2}{\eps^2} +\log \frac{n}{\delta})$ samples from $P$ in total. This completes the proof of the theorem.
\end{proof}

\subsection{Non-realizable structure learning lower bound}\label{sec:nonrealizablelb}

Now we show the dependence on $n$ and $\eps$ on the sample complexity of our structure learning result from \Cref{theo:treelearningnonrealizableintro} is tight.

\treelearningnonrealizablelb*

We will first prove the lower bound for $n=3$. Then we will generalize the lower bound for $n=3 \ell$ for some positive integer $\ell$.

Let us assume $B_1, B_2, U_1, V_1, W_1, U_2, V_2, W_2 \sim N(0,1)$ be eight iid random bits. The first tree $R_1$ is defined as follows:
\begin{align*}
 X_1 & \gets \paren{1+ \frac{\eps}{n}} \cdot B_1 + U_1 \\
 Y_1 & \gets \paren{1+ \frac{2\eps}{n}} \cdot B_1 + V_1 \\
 Z_1  & \gets \paren{1+ \frac{3 \eps}{n}} \cdot B_1 + W_1
\end{align*}


The second tree $R_2$ is defined as: 
\begin{align*}
X_2 & \gets \paren{1+ \frac{\eps}{n}} \cdot B_2 + U_2 \\
Y_2 & \gets \paren{1+ \frac{3 \eps}{n}} \cdot B_2 + V_2 \\
Z_2 & \gets \paren{1+ \frac{2\eps}{n}} \cdot B_2 + W_2
\end{align*}



Now let us consider the following observation.
\begin{obs}\label{obs:nonrealizableklbound}
\begin{itemize}
    \item[(i)] $\kl(R_1||R_2)=\Theta(\frac{\eps^2}{n^2})$. Similarly, we also have $\kl(R_2||R_1) = \Theta(\frac{\eps^2}{n^2})$.
    
    \item[(ii)] $I(X_1;Z_1) -I(X_1;Y_1) \geq \eps/3$.
\end{itemize}    
\end{obs}

\begin{proof}

\begin{itemize}
    \item[(i)]

From \Cref{obs:klhellingerbounds}, we know that the $\kl$ between $R_1$ and $R_2$ can be expressed as follows:    
\begin{equation}\label{eqn:klnonrealizable}
\kl(R_1 || R_2) = \frac{1}{2} \left(\tr(\Sigma_2^{-1} \Sigma_1) -3 + \ln\paren{\frac{\det(\Sigma_2)}{\det(\Sigma_1)}}\right)    
\end{equation}

Now let us consider the covariance matrix of $R_1$ below:
$$
\Sigma_1 = 
\begin{bmatrix}
(1+ \frac{\eps}{n})^2+1 & (1 + \frac{\eps}{n})(1 +  \frac{2\eps}{n}) & (1 + \frac{\eps}{n})(1+ \frac{3 \eps}{n}) \\
(1 + \frac{\eps}{n})(1+ \frac{2 \eps}{n}) & (1 + \frac{2\eps}{n})^2 + 1 & (1 + \frac{2\eps}{n})(1+ \frac{3 \eps}{n}) \\
(1 + \frac{\eps}{n})(1+ \frac{3 \eps}{n}) & (1 + \frac{2\eps}{n})(1+ \frac{3 \eps}{n}) & (1+ \frac{3\eps}{n})^2 +1 
\end{bmatrix}
$$

Similarly, the covariance matrix of $R_2$ is:
$$
\Sigma_2 = 
\begin{bmatrix}
(1+ \eps/n)^2+1 & (1 + \eps/n)(1+ 3 \eps/n) & (1 + \eps/n)(1 + 2 \eps/n) \\
(1 + \eps/n)(1+ 3 \eps/n) & (1 + 3\eps/n)^2 + 1 & (1 + 2\eps/n)(1+3 \eps/n) \\
(1 + \eps/n)(1 + 2 \eps/n) & (1 + 2\eps/n)(1+ 3 \eps/n) & (1+ 2\eps/n)^2 +1
\end{bmatrix}
$$

Here $\det(\Sigma_1)= \det(\Sigma_2) = 14 (\eps/n)^2 + 12 \eps/n + 4$. 
Thus, following \Cref{eqn:klnonrealizable}, we have:
$$\kl(R_1||R_2) = \frac{27(\eps/n)^4 + 24(\eps/n)^3 + 6(\eps/n)^2}{28(\eps/n)^2 + 24\eps/n+8} \approx \Theta(\frac{\eps^2}{n^2})$$

Similarly, we can also show that $\kl(R_2||R_1) \approx \Theta(\frac{\eps^2}{n^2})$.

\item[(ii)]

Recall that 
\begin{equation}\label{eqn:nonrealizablemidiff}
I(X_1;Y_1) - I(X_1;Z_1) = \frac{1}{2} \log \left(\frac{1-\rho_{X_1Z_1}^2}{1-\rho_{X_1Y_1}^2}\right)   
\end{equation}

Now $\rho_{X_1Y_1}= \frac{\Sigma_{X_1,Y_1}}{\sqrt{\Sigma_{X_1,X_1}\Sigma_{Y_1,Y_1}}} = \frac{(1+ \eps/n)(1+ 2 \eps/n)}{\sqrt{((1+2\eps/n)^2+1)((1+\eps/n)^2+1)}}$.
So, we have:
\begin{eqnarray*}
I(X_1;Y_1) &=& -\frac{1}{2} \log (1-\rho_{X_1Y_1}^2) \\
&=& \frac{1}{2} \log \frac{2(2(\eps/n)^2 + 2 \eps/n + 1)((\eps/n)^2+ 2 \eps/n + 2)}{5 (\eps/n)^2 + 6 \eps/n + 3} \\
&=& \frac{1}{2} \log \frac{4 +12 \eps/n + 18 (\eps/n)^2 + \Theta((\eps/n)^3)}{3 + 6 \eps/n + 5 (\eps/n)^2} \\
&=& \frac{1}{2} \log \frac{1+ 3\eps/n + \frac{9}{2}(\eps/n)^2 + \Theta((\eps/n)^3)}{1 + 2 \eps/n + \frac{5}{3} (\eps/n)^2} + \frac{1}{2} \log \frac{4}{3} \\ 
&=& \frac{1}{2} \log \left(1 + \frac{\eps/n}{1+ 2 \eps/n + \frac{5}{3}\eps^2/n^2} + \Theta(\eps^2/n^2)\right) + \frac{1}{2} \log \frac{4}{3} \\
&=&\frac{1}{2} \log \left(1 + \eps/n \cdot (1+ 2 \eps/n + \frac{5}{3}\eps^2/n^2)^{-1} + \Theta(\eps^2/n^2)\right) + \frac{1}{2} \log \frac{4}{3} \\
&=& \frac{1}{2} \log \left(1+ \eps/n \cdot (1-2 \eps/n) + \Theta(\eps^2/n^2)\right) + \frac{1}{2} \log \frac{4}{3} \\
&=& \frac{\eps}{2n} + \frac{1}{2} \log \frac{4}{3} \pm \Theta\left(\frac{\eps^2}{n^2}\right)
\end{eqnarray*}

Similarly, $\rho_{Y_1Z_1}= \frac{\Sigma_{Y_1,Z_1}}{\sqrt{\Sigma_{Y_1,Y_1}\Sigma_{Z_1,Z_1}}} = \frac{(1+2 \eps/n)(1+3 \eps/n)}{\sqrt{((1+2\eps/n)^2+1)((1+3\eps/n)^2+1)}}$. So, we have:
\begin{eqnarray*}
I(Y_1;Z_1) &=& -\frac{1}{2} \log (1-\rho_{Y_1Z_1}^2) \\
&=& \frac{1}{2} \log \frac{2(2+ 6 \eps/n + 9 \eps^2/n^2)(1+ 2 \eps/n + 2 \eps^2/n^2)}{3+ 10 \eps/n + 13 \eps^2/n^2} \\
&=& \frac{1}{2} \log \frac{4 + 20 \eps/n + 50 \eps^2/n^2 + \Theta(\eps^3/n^3)}{3 + 10 \eps/n + 13 \eps^2/n^2} \\
&=& \frac{1}{2} \log \left(\frac{1+ 5 \eps/n + \frac{25}{2}\eps^2/n^2 + \Theta(\eps^3/n^3)}{1 + \frac{10}{3} \eps/n + \frac{13}{3} \eps^2/n^2}\right) + \frac{1}{2} \log \frac{4}{3} \\
&=& \frac{1}{2} \log \left(1 + \frac{\frac{5}{3}\eps/n}{1 + \frac{10}{3} \eps/n + \frac{13}{3} \eps^2/n^2} + \Theta(\eps^2/n^2)\right) + \frac{1}{2} \log \frac{4}{3} \\
&=& \frac{1}{2} \log \left(1+ \frac{5}{3}\eps/n \cdot (1 + \frac{10}{3} \eps/n + \frac{13}{3} \eps^2/n^2)^{-1} + \Theta(\eps^2/n^2)\right) + \frac{1}{2} \log \frac{4}{3} \\
&=& \frac{5\eps}{6n} + \frac{1}{2} \log \frac{4}{3} \pm \Theta\left(\frac{\eps^2}{n^2}\right)
\end{eqnarray*}

Finally, $\rho_{X_1Z_1}= \frac{\Sigma_{X_1,Z_1}}{\sqrt{\Sigma_{X_1,X_1}\Sigma_{Z_1,Z_1}}} = \frac{(1+\eps/n)(1+2\eps/n)}{\sqrt{((1+\eps/n)^2+1)((1+3\eps/n)^2+1)}}$. So, we have:
\begin{eqnarray*}
I(X_1;Z_1) &=& -\frac{1}{2} \log (1-\rho_{X_1Z_1}^2) \\
&=& -\frac{1}{2} \log \frac{3+ 8 \eps/n + 10 \eps^2/n^2}{(2+ 6 \eps/n + 9 \eps^2/n^2)(2+ 2 \eps/n + \eps^2/n^2)} \\
&=& \frac{1}{2} \log \frac{4 + 16 \eps/n + 32 \eps^2/n^2 + \Theta(\eps^3/n^3)}{3+ 8 \eps/n + 10 \eps^2/n^2} \\
&=& \frac{1}{2} \log \left(\frac{1+ 4 \eps/n + 8 \eps^2/n^2 + \Theta(\eps^3/n^3)}{1+ \frac{8}{3} \eps/n + \frac{10}{3}\eps^2/n^2}\right) + \frac{1}{2} \log \frac{4}{3} \\
&=& \frac{1}{2} \log \left(1+ \frac{\frac{4}{3}\eps/n}{1+ \frac{8}{3}\eps/n + \frac{10}{3}\eps^2/n^2} + \Theta(\eps^2/n^2)\right) + \frac{1}{2} \log \frac{4}{3} \\
&=&\frac{1}{2} \log \left(1+ \frac{4}{3}\eps/n \cdot (1+ \frac{8}{3}\eps/n + \frac{10}{3}\eps^2/n^2)^{-1} + \Theta(\eps^2/n^2)\right) + \frac{1}{2} \log \frac{4}{3} \\
&=& \frac{4 \eps}{6n} + \frac{1}{2} \log \frac{4}{3} \pm \Theta\left(\frac{\eps^2}{n^2}\right)
\end{eqnarray*}
\end{itemize}
\end{proof}

Thus, in the tree $R_1$, $I(Y_1;Z_1) > I(X_1;Z_1) > I(X_1;Y_1)$. Similarly, in the tree $R_2$, we have 
$I(Y_2;Z_2) > I(X_2;Y_2) > I(X_2;Z_2)$.
Moreover, $I(X_1;Z_1)- I(X_1;Y_1) \geq \frac{\eps}{10n}$.

The best structure tree $R_1$ will be of the form $Y-Z-X$ and similarly, the best structured tree $R_2$ will be of the form $X-Y-Z$. Thus, the edge $Y-Z$ will be present in both $R_1$ and $R_2$, and either $X-Z$ or $X-Y$ will be present, depending on the tree. Equivalently, we can represent each tree $R_1$ or $R_2$ as a Boolean bit, indicating whether a particular edge (without loss of generality, assume the edge $X-Z$) is present. So, if the tree is $R_1$ we will set the bit $1$, and if the tree is $R_2$, we set the bit to $0$.

Now we will construct a collection of distributions $\mathcal{D}$ over $\{0,1\}^n$ by choosing $n$ building blocks independent of each other, each is either $R_1$ or $R_2$. Note that, from the construction, $\size{\mathcal{D}}=2^n$.
Now consider a subset of $\dfar \subset \mathcal{D}$ such that if we take any two distributions $D_1, D_2 \in \dfar$ such that $\mathsf{Ham}(D_1,D_2) \geq n/5$.

\begin{obs}
\begin{itemize}
    \item[(i)] $\size{\mathcal{D}_{\mathsf{far}}}= 2^{\Theta(n)}$.

    \item[(ii)] Consider any two arbitrary distribution $D_1, D_2 \in \mathcal{D}_{\mathsf{far}}$. Then $\kl(D_1||D_2) \leq \eps^2/n$. Also, $\kl(D_2||D_1) \leq \eps^2/n$.
\end{itemize}
\end{obs}

\begin{proof}

\begin{itemize}
    \item[(i)] This follows from coding theory: \emph{Gilbert-Varshamov Bound}, see \citep[Section 4.2]{guruswami2012essential} for reference.

    \item[(ii)] Consider any two arbitrary distributions $D_1, D_2 \in \mathcal{D}_{\mathsf{far}}$. Note that $D_1$ and $D_2$ differ in $n/2$ bits, that is, $D_1$ and $D_2$ are different in $n/2$ building blocks. Without loss of generality, let us assume that the last $n/2$ bits are different. Since the building blocks of the distributions in $\mathcal{D}_{\mathsf{far}}$ are chosen to be either $R_1$ or $R_2$, independently with equal probability, we can write the following:
    \begin{eqnarray*}
     \kl(D_1||D_2) &=& \sum_{i=1}^n \kl(R_1||R_2) \\
     &\leq& n \cdot \frac{\eps^2}{n^2}~~~~~~~~~~~~~~~~~~~~~~~~~~~[\mbox{From \Cref{obs:nonrealizableklbound}}]\\
     &=& \frac{\eps^2}{n}
    \end{eqnarray*}
    
\end{itemize}
    
\end{proof}

Similarly, we can also say that $\kl(D_2||D_1) \leq \eps^2/n$.
Now we are ready to apply Fano's inequality (\Cref{lem:fanoinequality}) to obtain the desired testing lower bound. By \Cref{lem:fanoinequality}, we can say that to test correctly with probability at least $9/10$, we need at least
$$\Omega\left(\frac{\log \size{\mathcal{D}_{\mathsf{far}}}}{\max\{\kl(D_1||D_2), \kl(D_2||D_1)\}}\right)$$

Since $\size{\mathcal{D}_{\mathsf{far}}}= 2^{\Theta(n)}$, and $\max\{\kl(D_1||D_2), \kl(D_2||D_1)\} \leq \eps^2/2n$, we need $\Omega(\frac{n^2}{\eps^2})$ samples to correctly test with probability at least $9/10$.

\begin{proof}[Proof of \Cref{theo:treelearningnonrealizablelbintro}] 

{We consider a random distribution $P \in \dfar$. Any algorithm that guesses $\wh{D}$ correctly with probability at least $9/10$, needs $\Omega(n^2\eps^{-2})$ samples.}  For every $i \in [n]$, let $(X_i,Y_i,Z_i)$ be the random variables for the $i$-th block. We will show that any algorithm $\mathcal{A}$ that returns an $\eps/200$-approximate tree will be able to correctly guess the true tree and therefore would need $\Omega(n^2\eps^{-2})$ samples. 

Let $T$ be the optimal tree structure and $D$ be the corresponding tree-structured distribution for $P$. Suppose $\mathcal{A}$ returns a tree structure $\wh{T}$ whereas the optimal tree structure is $T$. Let $$\wh{D}=\arg \min\limits_{\textrm{tree distributions }P \ \textrm{on } \widehat{T}} \kl(D||P).$$ Firstly, without loss of generality, we can remove the cross edges from $\wh{T}$ which connect across different blocks. Note that our true distribution belongs to $\dfar$ and does not have any such cross edges. This follows due to the fact that the blocks are independent of each other, hence the mutual information between the corresponding nodes is $0$, thereby, the weights of the every cross edges are $0$ as well.

Now, we guess the random tree using $\wh{T}$. Consider the case when $\wh{T}$ excludes an edge $(Y_i,Z_i)$ for $k$ blocks among all the $n$ blocks, and in the remaining $(n-k)$ blocks, the edge $(Y_i,Z_i)$ is present, for some $k \in [n]$. For each of these $k$ blocks, we will include the $(Y_i,Z_i)$ edge, and remove the edge $(X_i,Y_i)$. This can only make the tree approximation parameter smaller. Note that the corresponding distribution $D_1$ after these modification steps is in $\cD$.

Our guess for the unknown distribution $D$ will be $$D'=\arg \min\limits_{P \in \dfar} \ham(D_1,P).$$  Note that $\ham(D_1,D')\le n/20$ as, otherwise, $\ham(D_1,D)> n/20$, and $\kl(\wh{D}||D)> \frac{n}{20} \cdot \frac{\eps}{10n}=\frac{\eps}{200}$, leading to a contradiction. Henceforth, we will assume $\ham(\wh{D},D')\le n/20$. Let $T'$ be the tree structure for $D'$. By our construction, we can say that  $\ham(D',D_1) < n/5$.

In that case, by our construction, $\ham(D_1,D'')\ge \frac{n}{5}-\frac{n}{20} \ge \frac{3n}{20}>\frac{n}{20}$ for every $D''\in \dfar$ such that $D'\neq D''$. Therefore, there is a unique distribution in $\dfar$, which is exactly $D$, that is $\frac{n}{20}$-close to $\wh{D}$ in Hamming distance. This will be our guess for the unknown random distribution.
 
Combining all the cases, we see that using any $\eps/200$-approximate structure learning algorithm, we can correctly guess the random distribution without requiring any further samples. Therefore, the claimed lower bound holds.
\end{proof}

\subsection{Learning Tree Structured Distributions: realizable case}\label{sec:realizable}

Now, we turn to the realizable case, where $P$ is known to be tree-structured. In this case, we give a quadratically-improved sample complexity of structure learning.

\treelearningrealizable*

\begin{proof}[Proof of \Cref{theo:treelearningrealizableintro}]
Similar to the non-realizable setting, we will prove this above result under the zero-mean assumption. Using~\Cref{theo:treelearningrealizablenonzeromean}, we will get a result for the non-zero mean case by doubling the number of samples.

Let us assume that the unknown distribution $P$ is supported over the true (unknown) tree $T$ and Chow-Liu algorithm outputs a tree $\widehat{T}$. Now we will show that $\kl(P||P_{\widehat{T}}) \leq 2 \eps$, and finally we will re-scale $\eps$ to obtain the desired result.

\begin{cl}
$\kl(P||P_{\widehat{T}}) \leq 2 \eps$.    
\end{cl}

\begin{proof}
We will prove this by contradiction. Let us assume that $\kl(P||P_{\widehat{T}}) > 2 \eps$. Let us denote the edges of $T^*$ and $\widehat{T}$ as $\langle e_i \rangle_{i=1}^{n-1}$ and $\langle f_i \rangle_{i=1}^{n-1}$, respectively. Let us assume that the symmetric difference between the edges of $T$ and $\wh{T}$ have $\ell\le (n-1)$ edges. Following \Cref{lem:spanningtreelemma}, let us assume that $\{\langle e_i,f_i \rangle\}_{i=1}^{\ell}$ be the pairing of the edges of $T$ and $\widehat{T}$. 
 
Then we have the following:
\begin{eqnarray*}
&&\kl(P||P_{\widehat{T}}) > 2 \eps \\
&& \mathsf{wt}_P(T) - \mathsf{wt}_P(\widehat{T}) > 2 \eps \\
&& \sum_{i=1}^{\ell} \left(I(e_i) - I(f_i)\right) > 2 \eps
\end{eqnarray*}

Thus there exists an $i \in [\ell]$ such that $I(e_i) - I(f_i) > \frac{2 \eps}{\ell} \geq \frac{2 \eps}{(n-1)}$
holds, where $e_i=(X_i,Y_i) \in T\setminus\wh{T}$ and $f_i=(W_i,Z_i) \in \wh{T}\setminus {T}$. Then we have the following: 
\begin{align*}
\frac{2 \eps}{n-1} &< I(e_i) - I(f_i)   \\
&= I(X_i;Y_i) - I(W_i;Z_i) \\
&= I(X_i;Y_i) - I(Y_i;W_i) + I(Y_i;W_i) - I(W_i;Z_i)\\
&= I(X_i;Y_i \mid W_i) - I(Y_i;W_i \mid X_i) + I(Y_i;W_i \mid Z_i) - I(Z_i;W_i \mid Y_i),
\end{align*}
the last step follows from using the chain rule of CMI (\Cref{obs:michainrule}).

Note that in the true tree $T$, $X_i$ and $Y_i$ lies on the unique paths between $Y_i,W_i$ and $Z_i,W_i$ respectively. From~\Cref{lem:treeci} this implies the following:
\begin{equation}\label{eqn:condrealizablezero}
I(Y_i;W_i \mid X_i) = I(Z_i;W_i \mid Y_i) =0  
\end{equation}

Thus, we have:
\begin{equation}
I(X_i;Y_i \mid W_i) + I(Y_i;W_i \mid Z_i)  > \frac{2 \eps}{n-1}    
\end{equation}

This implies that either $I(X_i;Y_i \mid W_i) > \frac{\eps}{n-1}$, or $I(Y_i;W_i \mid Z_i) > \frac{\eps}{n-1}$. Following \Cref{lem:cmiubintro} (called with $\eps/(n-1)$ in place of $\eps$), we can say that either 
$\wh{I}(X_i;Y_i \mid W_i) > \frac{\eps}{8(n-1)}$ or $\wh{I}(Y_i;W_i \mid Z_i) > \frac{\eps}{8(n-1)}$.

In a similar manner, from \Cref{lem:cmiubintro} and \Cref{eqn:condrealizablezero}, we can say that $\wh{I}(Y_i;W_i \mid X_i) \leq \eps/20(n-1)$ and  $\wh{I}(Z_i;W_i \mid Y_i) \leq \eps/20(n-1)$.

Now let us consider the tree $T' = \widehat{T} \cup \{e_i\} \setminus \{f_i\}$. Then we have the following:
\begin{eqnarray*}
\wh{\mathsf{wt}}(T') - \wh{\mathsf{wt}}(\wh{T}) &=& \wh{I}(e_i) - \wh{I}(f_i) \\
&=& \wh{I}(X_i;Y_i \mid W_i) - \wh{I}(Y_i;W_i \mid X_i) \\ && ~~~~~~~~~~~~~~~~~~ + \wh{I}(Y_i;W_i \mid Z_i) - \wh{I}(Z_i;W_i \mid Y_i) \\
&>& \frac{\eps}{8(n-1)} - \frac{\eps}{10(n-1)} \\
&> 0&
\end{eqnarray*}

The second line follows from the empirical chain rule of conditional mutual information, as proved in \Cref{lem:cmichainruleempintro}. However, this implies that $T' = \widehat{T} \cup \{e_i\} \setminus \{f_i\}$ is a better tree as compared to $\wh{T}$, which is a contradiction!

Combining, we can say that $\kl(P||P_{\widehat{T}}) \leq 2 \eps$. This completes the proof of the claim.
\end{proof}

Now let us consider the sample complexity. As mentioned before, we will rescale $2 \eps$ to $\eps$ to obtain an $\eps$-approximate tree. However, this will not change the sample complexity of our algorithm by not more than a factor of $2$. As we will be calling \Cref{lem:cmiubintro} with parameter $\eps/(n-1)$ in place of $\eps$, the overall sample complexity would be $\Oh(\frac{n}{\eps} \log \frac{n}{\delta})$. This completes the proof of the theorem.
\end{proof}

\subsection{Realizable structure learning lower bound}\label{sec:realizablelb}

Now we prove the tightness of our result for the realizable structure learning problem.

\treelearningrealizablelb*

Similar to the non-realizable lower bound, we will first prove the lower bound for $n=3$ nodes. Then we will extend it for $n=3 \ell$, for some positive integer $\ell$.

Let $B_1, B_2, U_1, V_1, W_1, U_2, V_2, W_2 \sim N(0,1)$ be eight iid bits. The first tree $R_1$ is defined as follows:
\begin{align*}
    X_1  & \gets U_1\\
    Y_1  & \gets V_1 + B_1 \\
    Z_1 & \gets \sqrt{\frac{\eps}{n}} X_1 + W_1 + B_1
\end{align*}

The second tree $R_2$ is defined as follows:
\begin{align*}
    X_2 & \gets U_2 \\
    Y_2 & \gets \sqrt{\frac{\eps}{n}} X_2 + V_2 + B_2 \\
    Z_2 & \gets W_2 +B_2
\end{align*}

Now we have the following two claims.
\begin{cl}\label{cl:realizableklbound}
 \begin{itemize}
     \item[(i)] $\kl(R_1||R_2) = \eps/n$. Similarly, $\kl(R_2||R_1) = \eps/n$.
     \item[(ii)] $I(X_1;Y_1)- I(X_1;Z_1) \geq \Theta(\eps/n)$.
 \end{itemize}   
\end{cl}

\begin{proof}

Let us first consider the covariance matrix of $R_1$ as follows:

$$\Sigma_1 = 
\begin{bmatrix}
1 & 0 & \sqrt{\frac{\eps}{n}} \\
0 & 2 & 1 \\
\sqrt{\frac{\eps}{n}} & 1 & 2 + \frac{\eps}{n} 
\end{bmatrix}$$

$$\Sigma_2 = 
\begin{bmatrix}
1 & \sqrt{\frac{\eps}{n}} & 0 \\
\sqrt{\frac{\eps}{n}} & (2 + \frac{\eps}{n}) & 1 \\
0 & 1 & 2 
\end{bmatrix}$$

Note that $\det(\Sigma_1)= \det(\Sigma_2)=3$.

\begin{itemize}
    \item[(i)]

From \Cref{obs:klhellingerbounds}, we know the following:
\begin{equation}
\kl(R_1 || R_2) = \frac{1}{2} \left(\tr(\Sigma_2^{-1} \Sigma_1) -3 + \ln\paren{\frac{\det(\Sigma_2)}{\det(\Sigma_1)}}\right)
\end{equation}

Putting the values of $\Sigma_1$ and $\Sigma_2$ from above, we have:
\begin{eqnarray*}
\kl(R_1 || R_2) = \eps/n.   
\end{eqnarray*}

Similarly, we can also show that $\kl(R_2 || R_1) = \eps/n$. This completes our proof.

\item[(ii)]
As before, recall that $\rho_{X_1,Z_1} = \frac{\Sigma_{X_1,Z_1}}{\sqrt{\Sigma_{X_1,X_1}\Sigma_{Z_1,Z_1}}}$. So, $\rho_{X_1,Z_1}=\frac{\sqrt{\eps/n}}{\sqrt{2+ \eps/n}}$. Similarly, $\rho_{X_1Y_1}= 0$, and  $\rho_{Y_1,Z_1} = \frac{1}{\sqrt{2(2+\eps/n)}}$. Moreover, note that
\begin{eqnarray*}
I(X_1;Z_1) &=& -\frac{1}{2} \log \paren{1-\rho_{X_1,Z_1}^2} \\ &=& - \frac{1}{2} \log \paren{1- \frac{\eps/n}{2+\eps/n}} \\ &=& - \frac{1}{2} \log \paren{\frac{2}{2+\eps/n}} \\ &=&\frac{1}{2} \log \paren{1+ \frac{\eps}{2n}} \\ &>& \frac{\eps}{16n}   
\end{eqnarray*}

Similarly,
\begin{eqnarray*}
I(Y_1;Z_1) &=& -\frac{1}{2} \log \paren{1-\rho_{Y_1,Z_1}^2} \\ &=& - \frac{1}{2} \log \paren{1- \frac{1}{4+2\eps}} \\ &=& - \frac{1}{2} \log \paren{\frac{3+ 2 \eps}{4+2\eps}} \\ &=&\frac{1}{2} \log \paren{1+ \frac{1}{3 + 2\eps}} \\ &\approx& \Theta(1)   
\end{eqnarray*}    

Finally, as $\rho_{X_1Y_1}=0$, $I(X_1;Y_1)=0$.

Thus,  from above, we can say that
\begin{equation}\label{eqn:realizalemidiff}
I(X_1;Z_1)- I(X_1;Y_1) \geq \frac{\eps}{16n}    
\end{equation}

\end{itemize}
    
\end{proof}

Thus,  the true structure tree $R_1$ will be of the form $Y-Z-X$ and similarly, the true structured tree $R_2$ will be of the form $X-Y-Z$. Thus, the edge $Y-Z$ will be present in both $R_1$ and $R_2$, and either $X-Z$ or $X-Y$ will be present, depending on the tree. Equivalently, we can represent each tree $R_1$ or $R_2$ as a Boolean bit, indicating whether a particular edge (without loss of generality, assume the edge $X-Z$) is present. So, if the tree is $R_1$ we will set the bit $1$, and if the tree is $R_2$, we set the bit to $0$.

Now we will construct a collection of distributions $\mathcal{D}$ over $\{0,1\}^n$ by choosing $n$ building blocks independent of each other, each is either $R_1$ or $R_2$. Note that, from the construction, $\size{\mathcal{D}}=2^n$.
Now consider a subset of $\dfar \subset \mathcal{D}$ such that if we take any two distributions $D_1, D_2 \in \dfar$ such that $\mathsf{Ham}(D_1,D_2) \geq n/5$.

We will start with the following observation.

\begin{obs}
\begin{itemize}
    \item[(i)] $\size{\dfar}= 2^{\Theta(n)}$.

    \item[(ii)] Consider any two arbitrary distribution $D_1, D_2 \in \dfar$. Then $\kl(D_1||D_2) \leq \eps$. Similarly, $\kl(D_2||D_1) \leq \eps$.
\end{itemize}
\end{obs}

\begin{proof}

\begin{itemize}
    \item[(i)] This follows from coding theory: the Gilbert-Varshamov bound, see \citep[Section 4.2]{guruswami2012essential} for reference.

    \item[(ii)] Consider any two arbitrary distributions $D_1, D_2 \in \dfar$. Note that $D_1$ and $D_2$ differ in $n_2$ bits, that is, $D_1$ and $D_2$ are different in $n/2$ building blocks. Without loss of generality, let us assume that the last $n/2$ bits are different. Since the building blocks of the distributions in $\dfar$ are chosen to be either $R_1$ or $R_2$, independently with equal probability, we can write the following:
    \begin{eqnarray*}
     \kl(D_1||D_2) &=& \sum_{i=1}^n \kl(R_1||R_2) \\
     &\leq& n \cdot \frac{\eps}{n}~~~~~~~~~~~~~~~~~~~[\mbox{From \Cref{cl:realizableklbound}}]\\
     &=& \eps
    \end{eqnarray*}
    
\end{itemize}
    
\end{proof}

Similarly, we can also show that $\kl(D_2||D_1) \leq \eps$.
Now we are ready to apply Fano's inequality (\Cref{lem:fanoinequality}) to obtain the desired testing lower bound. By \Cref{lem:fanoinequality}, we can say that to test correctly with probability at least $9/10$, we need at least
$$\Omega\left(\frac{\log \size{\dfar}}{\max\{\kl(D_1||D_2), \kl(D_2||D_1)\}}\right)$$

Since $\size{\dfar}= 2^{\Theta(n)}$, and $\max\{\kl(D_1||D_2), \kl(D_2||D_1)\} \leq \eps/2$, we need $\Omega(\frac{n}{\eps})$ samples to correctly test with probability at least $9/10$.

\begin{proof}[Proof of \Cref{theo:treelearningrealizablelbintro}]
We consider a random distribution $D\in \dfar$. Any algorithm that guesses $\wh{D}$ correctly with at least 9/10 probability requires $\Omega(n\eps^{-1})$ samples. For every $i \in [n]$, let $(X_i,Y_i,Z_i)$ be the random variables for the $i$-th block. We will show that any algorithm $\mathcal{A}$ that returns an $\eps/320$-approximate tree will be able to correctly guess the true tree and therefore would need $\Omega(n\eps^{-1})$ samples.

To see this, suppose $\mathcal{A}$ returns a tree structure $\wh{T}$ whereas the true tree structure is $T$. Let $\wh{D}=\arg \min\limits_{\textrm{tree distributions }P \ \textrm{on } \widehat{T}} \kl(D||P)$. Firstly, without loss of generality, we can remove the edges from $\wh{T}$ which connect across different blocks. Note that our true distribution belongs to $\dfar$ and does not have any such cross edges. The above follows due to the fact that the blocks are independent of each other, hence the mutual information between the corresponding nodes is $0$, thereby, the weights of the every cross edges are $0$ as well.

Now consider the case when $\wh{T}$ excludes an edge $(Y_i,Z_i)$ for some block $i \in [n]$. In that case, $\kl(D||\wh{D}) \ge \kl(D_i||\wh{D}_i) = I(Y_i;Z_i) - \max\{I(X_i;Z_i), I(X_i;Y_i)\} \ge \Theta(1)$ and $\wh{D}$ is no longer $\eps/320$-close,
where $D_i$ denotes the marginal distribution of the $i$-th block of $D$. Henceforth, we will assume that $\wh{D}\in \mathcal{D}$ without loss of generality.

Our guess for the unknown distribution $D$ will be $D'=\arg \min\limits_{P \in \dfar} \ham(\wh{D},P)$. Let $T'$ be the tree structure for $D'$. Note that $\ham(\wh{D},D')\le n/20$ as, otherwise, $\kl(\wh{D}||D)> \frac{n}{20} \cdot \frac{\eps}{16n}=\frac{\eps}{320}$, leading to a contradiction. Henceforth, we will assume $\ham(\wh{D},D')\le n/20$.

In that case, by our construction, $\ham(\wh{D},D'')\ge \frac{n}{5}-\frac{n}{20} \ge \frac{3n}{20}>\frac{n}{20}$ for every $D''\in \dfar$ such that $D'\neq D''$. Therefore, there is a unique distribution in $\dfar$, which is exactly $D$, that is $\frac{n}{20}$-close to $\wh{D}$ in Hamming distance. This will be our guess for the unknown random distribution.

Combining all the cases, we see that using any $\eps/320$-approximate structure learning algorithm, we can correctly guess the random distribution without requiring any further samples. Therefore, the claimed lower bound holds.
\end{proof}

\section{Experimental results}\label{sec:experiments_full}

In this section, we will present our experimental results. We start by describing our experimental setup in \Cref{sec:experiment_setup}. In \Cref{sec:experiment:chowliu}, we will present our results on Gaussian Chow-Liu algorithm where we compare the tolerance parameter ($\eps$) and the number of samples necessary ($m^*$), along with linear regression plots for comparing $\log m$ vs $\log 1/\eps$. In particular, in \Cref{sec:experiment:chowliu:realizable}, we demonstrate our result on the realizable setting, and in \Cref{sec:experiment:chowliu:nonrealizable},we present our results for the non-realizable setting. Later in \Cref{sec:experiment:chowliuglasso} and \Cref{sec:experiment:chowliuclime}, we compare the Gaussian Chow-Liu result in the realizable setting with Graphical Lasso (GLASSO) and Constrained $\ell_1$ Inverse Covariance Estimation (CLIME), two methods widely used in practice. We compare these two methods for $4$ settings of the tolerance parameter $\eps$.  Then, in \Cref{sec:experiment:mitest}, we discuss our results on Mutual Information testing, where in \Cref{sec:experiment:mitest:realizable}, we study the realizable case, and in \Cref{sec:experiment:mitest:nonrealizable}, we present our results for the additive estimation of mutual information.
Finally, in \Cref{sec:experiment_cmi}, we discuss our results on Conditional Mutual Information testing, where in \Cref{sec:experiment:cmitest:realizable}, we study the realizable setting and in \Cref{sec:experiment:cmitest:nonrealizable}, we present our results for additive estimation of conditional mutual information.

\subsection{Experimental Setup}\label{sec:experiment_setup}

All experiments are carried out on a personal computer Mac Book Air with an M2 chip having a RAM of 8GB. All the experiments are conducted using synthetic datasets constructed using our hard constructions from the lower bound theorems. Experiments took at most 4 hours each.

\color{black}

\subsection{Chow-Liu : $\eps$ vs $m^*$}\label{sec:experiment:chowliu}

In this experiment, we investigate the relationship between the number of samples required and the desired level of tolerance ($\eps$) for the Kullback-Leibler (KL) divergence to converge to zero. Specifically, we aim to determine the minimum number of samples necessary ($m^*$) for the KL divergence between the original structure distribution and the predicted tree distribution to approach zero within a fixed tolerance $\varepsilon$. Let us start with the realizable setting.

\subsubsection{Realizable setting}\label{sec:experiment:chowliu:realizable}

Let $U_1, V_1, W_1 \sim N(0,1)$ be three iid bits. Then, the realizable tree $R_1$ is defined as follows:
\begin{align*}
    Y_1  & \gets U_1\\
    Z_1  & \gets 0.5Y_1 + W_1 \\
    X_1 & \gets \sqrt{\eps} Z_1 + V_1
\end{align*}

For the fixing of tolerance ($\eps$), we are looking at the minimum number of samples ($m$) at which the KL-Divergence between the original tree and the predicted tree using Chow-Liu algorithm is less than or equal to $\frac{\eps}{4}$ with probability $\geq 0.95$. Here, we are considering 1000 trials for every fixing of $\eps$. Then, note the least value of $m$ (i.e. $m^*$) for which the KL between predicted tree distribution and original tree distribution is less than or equal to $\frac{\eps}{4}$.

The relation between $m^*$ and $\eps$ for the realizable case with tree structure $R_1$ is presented in \Cref{fig:prob1_6_2}. Here, $m^*$ is approximately inversely proportional to $\eps$. After fitting the linear regression to the $\log(m^*)$ vs $\log\left(\frac{1}{\eps}\right)$ data, we obtain a line with a slope of $1.016$.

\begin{figure}[h]
    \centering
    \begin{minipage}{.5\textwidth}
        \centering
     \includegraphics[width=0.8\linewidth]{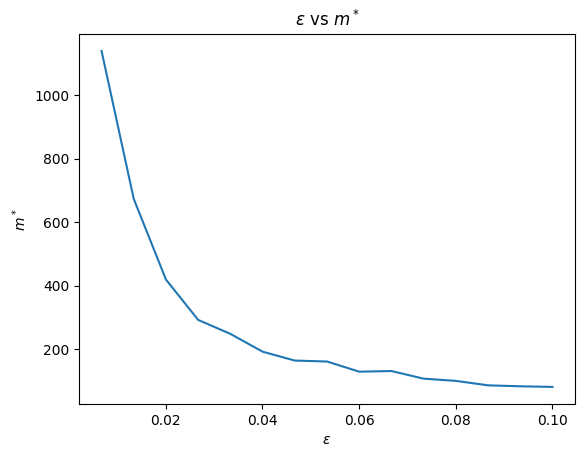}
        \caption{Proximity parameter vs No. of samples for the realizable case}
        \label{fig:prob1_6_2}
    \end{minipage}%
    \begin{minipage}{0.5\textwidth}
        \centering
    \includegraphics[width=0.8\linewidth]{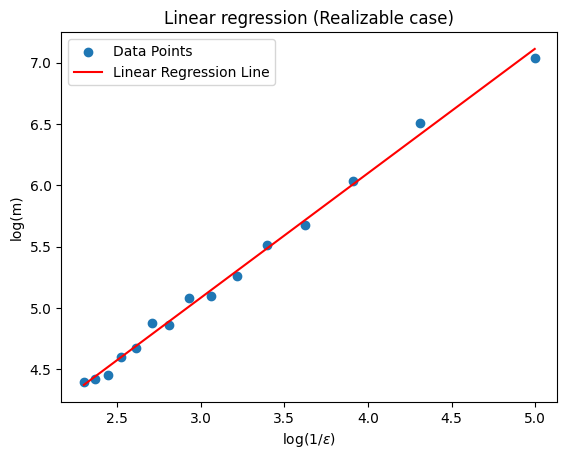}
        \caption{$\eps$ vs $m^*$ linear regression for realizable case }
        \label{fig:prob1_6_1}
    \end{minipage}
\end{figure}

\pagebreak
\subsubsection{Non-realizable setting}\label{sec:experiment:chowliu:nonrealizable}

Let us assume $B_1, U_1, V_1, W_1 \sim N(0,1)$ be four iid random bits. Then, the non-realizable structure $R_2$ is defined as follows:
\begin{align*}
 X_1 & \gets \paren{1+ \eps} \cdot B_1 + U_1 \\
 Y_1 & \gets \paren{1+ 2\eps} \cdot B_1 + V_1 \\
 Z_1  & \gets \paren{1+ 3\eps} \cdot B_1 + W_1
\end{align*}

For the fixing of tolerance ($\eps$), we are looking at the minimum number of samples ($m$) at which the KL-Divergence between the original graph and the predicted tree using the Chow-Liu algorithm is less than or equal to $\frac{\eps}{4}$ with probability $\geq 0.95$. Here, we are considering 1000 trials for every fixing of $\eps$. Then, noting the least value of $m$ (i.e. $m^*$) for which the KL between predicted tree distribution and original structure distribution is less than or equal to $\frac{\eps}{4}$. After fitting the linear regression to the $\log(m^*)$ vs $\log\left(\frac{1}{\eps}\right)$ data, we obtain a line with a slope of $1.916$. Here, $m^*$ is approximately inversely proportional to $\eps^2$.







\begin{figure}[!htb]
    \centering
    \begin{minipage}{.5\textwidth}
        \centering
     \includegraphics[width=0.8\linewidth]{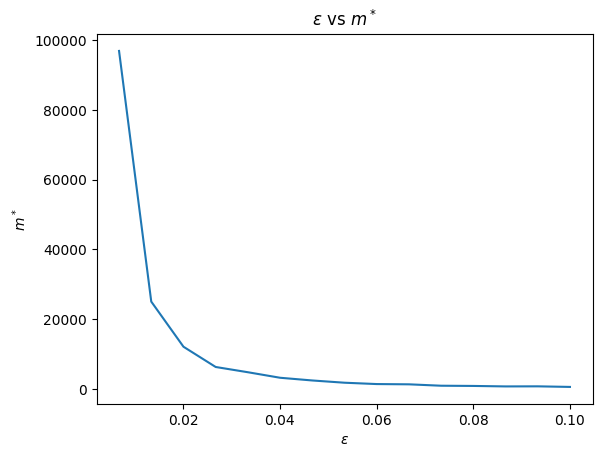}
        \caption{Proximity parameter vs No. of samples for the non-realizable case}
        \label{fig:prob1_7_2}
    \end{minipage}%
    \begin{minipage}{0.5\textwidth}
        \centering
    \includegraphics[width=0.8\linewidth]{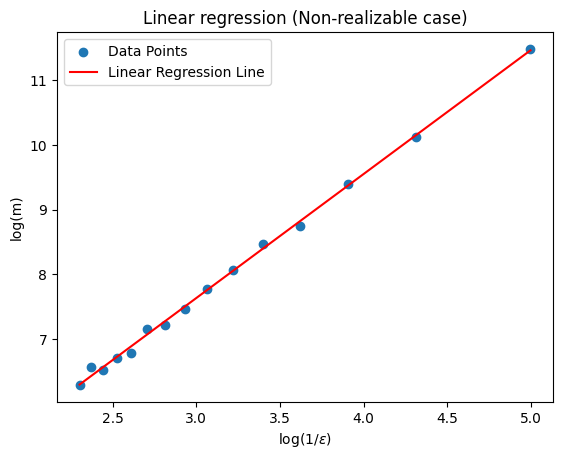}
        \caption{$\eps$ vs $m^*$ linear regression for non-realizable case}
        \label{fig:prob1_7_1}
    \end{minipage}
\end{figure}

\subsection{Chow-Liu and GLASSO Comparison}\label{sec:experiment:chowliuglasso}

In this subsection, we compare the results of our Gaussian Chow-Liu algorithm with Graphical Lasso (GLASSO), a popular algorithm used in the literature for estimation of covariance matrices. Let us start by briefly describing GLASSO.

\subsubsection{Graphical Lasso (GLASSO)}
The Graphical Lasso (GLASSO) is an algorithm for estimating sparse inverse covariance matrices in Gaussian graphical models, especially useful for high-dimensional data where the number of variables \( p \) exceeds the number of observations \( n \). It estimates the precision matrix \( \Theta = \Sigma^{-1} \), revealing conditional dependencies between variables.
The precision matrix \( \Theta \) is estimated by solving the optimization problem stated in \cite{friedman2008sparse}:

\[
\hat{\Theta} = \underset{\Theta \succ 0}{\text{argmin}} \left\{ \text{tr}(S \Theta) - \log \det(\Theta) + \lambda \|\Theta\|_1 \right\}
\]

Here, \( S \) is the sample covariance matrix, \( \|\Theta\|_1 \) is the element-wise \( \ell_1 \)-norm of \( \Theta \), and \( \lambda \) is a regularization parameter controlling sparsity. The \(\ell_1\)-penalty induces sparsity in \( \Theta \).

\subsubsection{Chow-Liu and GLASSO Comparison experiments}
Firstly, we will set the original tree construction similar to \Cref{sec:experiment:chowliu}. Let $U, V, W \sim N(0,1)$ be three iid bits. Then, the realizable tree $R_1$ is defined as follows:
\begin{align*}
    Y  & \gets U\\
    Z  & \gets 0.5Y + W \\
    X & \gets \sqrt{\eps} Z + V
\end{align*}

Then, do the following experiment on it.
We used the GLASSO package from R for the experiments~\cite{glassopackage}.
GLASSO returns a precision matrix for the triangle graph with the nodes $X,Y,Z$. We simply remove the lowest entry of the returned precision matrix in absolute value to obtain the approximate tree structure and compare the latter with the Chow-Liu output.
We conduct 200 trials to ascertain the frequency of correct recovery for various combinations of fixed epsilon and sample size ($m$). This gives us a way to compare the performance of the two algorithms to determine which requires fewer samples to prefer the $X$-$Z$ edge over the $X$-$Y$ edge.
We report this frequency of correct recovery in \Cref{fig:glasso_r_1}--\Cref{fig:glasso_r_4} by the above two algorithms for
100 different fixings of samples for $4$ different values of $\eps: 0.1,0.01,0.001,0.0001$.

\begin{figure}[h!]
    \centering
    \begin{minipage}{.45\textwidth}
        \centering
        \includegraphics[width=\linewidth]{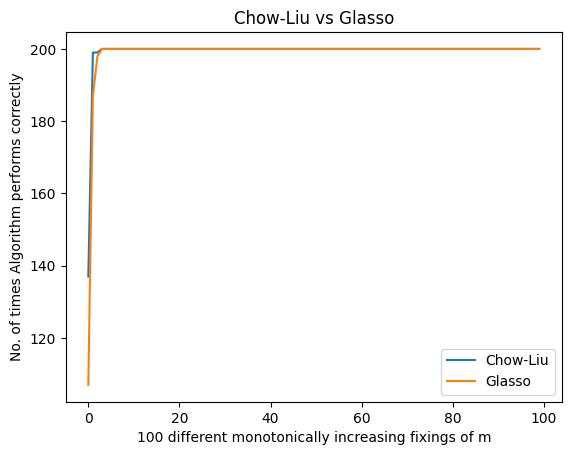}
        \caption{Chow-Liu vs GLASSO for $\eps = 0.1$}
        \label{fig:glasso_r_1}
    \end{minipage}
    \begin{minipage}{.45\textwidth}
        \centering     \includegraphics[width=\linewidth]{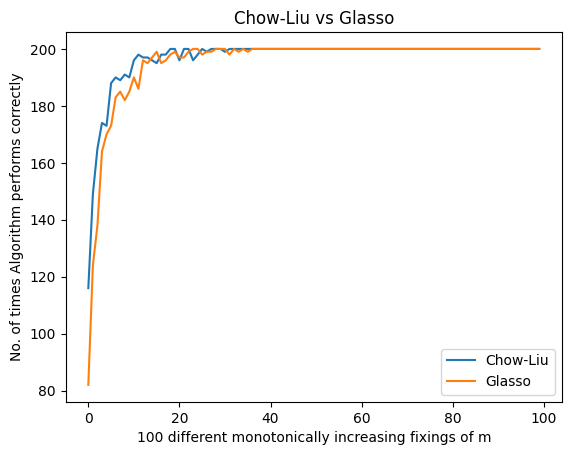}
        \caption{Chow-Liu vs GLASSO for $\eps = 0.01$}
        \label{fig:glasso_r_2}
    \end{minipage}
    \begin{minipage}{0.45\textwidth}
        \centering
    \includegraphics[width=\linewidth]{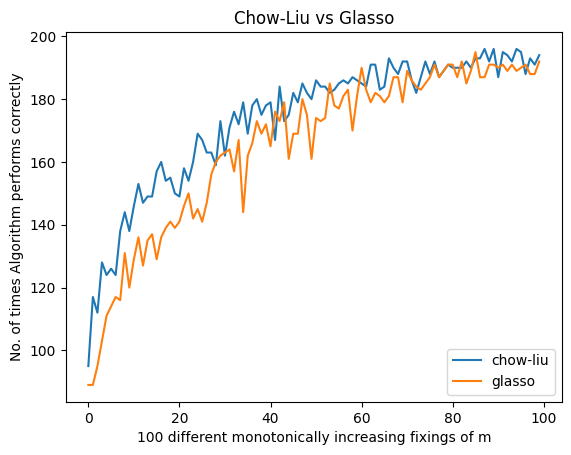}
        \caption{Chow-Liu vs GLASSO for $\eps = 0.001$}
        \label{fig:glasso_r_3}
    \end{minipage}
    \begin{minipage}{.45\textwidth}
        \centering
     \includegraphics[width=\linewidth]{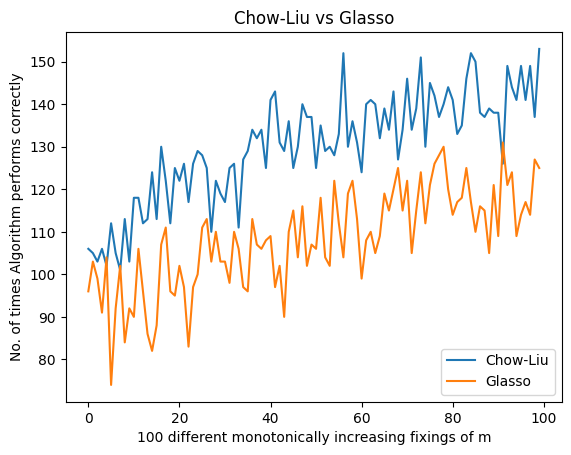}
        \caption{Chow-Liu vs GLASSO for $\eps = 0.0001$}
        \label{fig:glasso_r_4}
    \end{minipage}%
\end{figure}

Overall, we can conclude that, for the tree-structured distributions, the Chow-Liu algorithm performs better than the GLASSO algorithm in reconstructing the structure from the samples. 

\subsection{Chow-Liu and CLIME Comparison}\label{sec:experiment:chowliuclime}

In this subsection, we compare the results of our Gaussian Chow-Liu algorithm with CLIME, another popular algorithm used in the literature for estimation of covariance matrices. We start by briefly describing the method.

\subsubsection{Constrained $\ell_1$ Inverse Covariance Estimation (CLIME)}

Constrained $\ell_1$ Inverse Covariance Estimation (CLIME) is an algorithm for estimating inverse covariance matrices in Gaussian Graphical models. Similar to GLASSO, this is useful for high-dimensional data. It estimates the precision matrix \( \Theta = \Sigma^{-1} \), revealing conditional dependencies between variables. The precision matrix \( \Theta \) is estimated by solving the optimization problem stated in \cite{cai2011constrained}.
$$\min_{\Theta} ||\Theta||_1 \ \mbox{ s.t. } \ ||\hat{\Sigma} \Theta - I||_{\infty} \leq \lambda$$

\subsubsection{Chow-Liu and CLIME Comparison experiments}

We will set the original tree construction similar to \Cref{sec:experiment:chowliu}, as done for experiments with GLASSO. Let $U, V, W \sim N(0,1)$ be three iid bits. Then, the realizable tree $R_1$ is defined as follows:
\begin{align*}
    Y  & \gets U\\
    Z  & \gets 0.5Y + W \\
    X & \gets \sqrt{\eps} Z + V
\end{align*}

Then, we perform the following experiment on it.
We use the CLIME package from R for the experiments~\cite{climepackage}.
CLIME returns a precision matrix for the triangle graph with the nodes $X,Y,Z$. We simply remove the lowest entry of the returned precision matrix in absolute value to obtain the approximate tree structure and compare the latter with the Chow-Liu output.
We conduct 200 trials to ascertain the frequency of correct recovery for various combinations of fixed epsilon and sample size ($m$). This gives us a way to compare the performance of the two algorithms to determine which requires fewer samples to prefer the $X$-$Z$ edge over the $X$-$Y$ edge.
We report this frequency of correct recovery in \Cref{fig:clime_r_1}--\Cref{fig:clime_r_4} by the above two algorithms for
100 different fixings of samples for $4$ different values of $\eps: 0.1,0.01,0.001,0.0001$.

Similar to the case of GLASSO, we conclude that for the tree-structured distributions, the Chow-Liu algorithm performs better than the CLIME algorithm in reconstructing the structure from the samples.

\begin{figure}[hbt!]
    \centering
    \begin{minipage}{.45\textwidth}
        \centering
        \includegraphics[width=\linewidth]{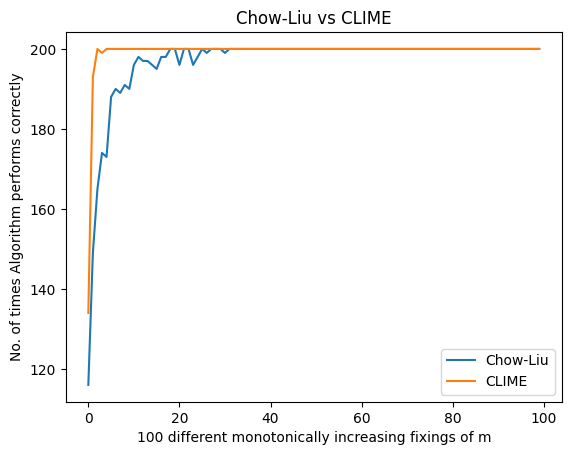}
        \caption{Chow-Liu vs CLIME for $\eps = 0.1$}
        \label{fig:clime_r_1}
    \end{minipage}
    \begin{minipage}{.45\textwidth}
        \centering     \includegraphics[width=\linewidth]{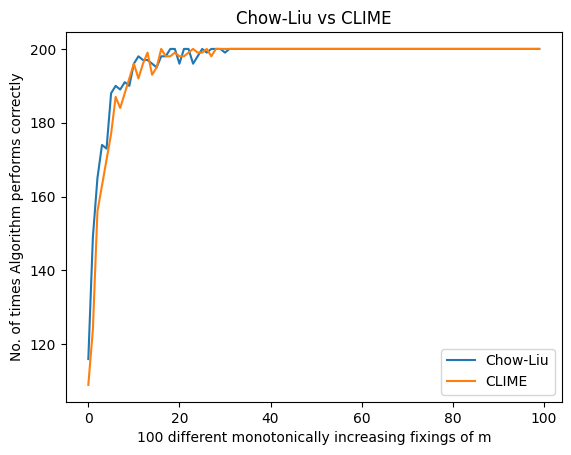}
        \caption{Chow-Liu vs CLIME for $\eps = 0.01$}
        \label{fig:clime_r_2}
    \end{minipage}
    \begin{minipage}{0.45\textwidth}
        \centering
    \includegraphics[width=\linewidth]{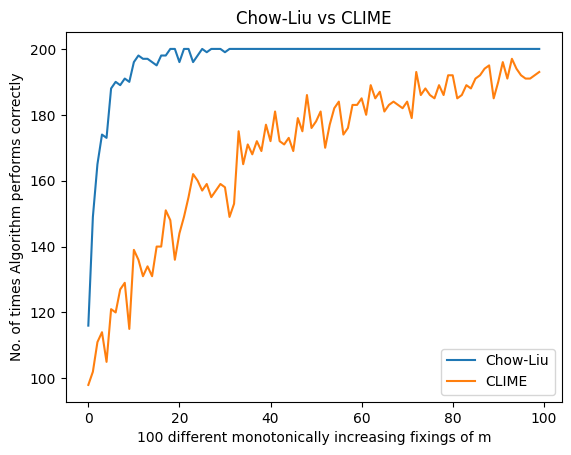}
        \caption{Chow-Liu vs CLIME for $\eps = 0.001$}
        \label{fig:clime_r_3}
    \end{minipage}
    \begin{minipage}{.45\textwidth}
        \centering
     \includegraphics[width=\linewidth]{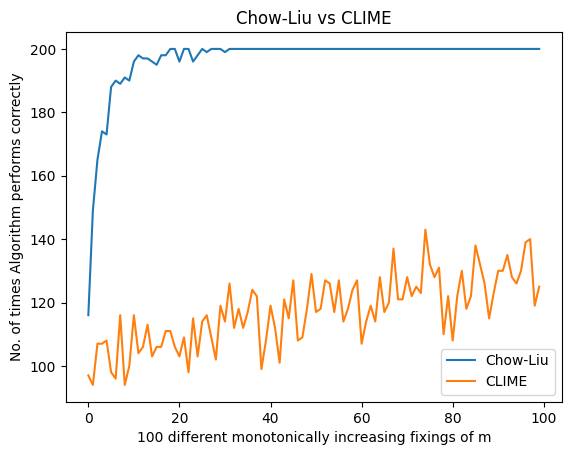}
        \caption{Chow-Liu vs CLIME for $\eps = 0.0001$}
        \label{fig:clime_r_4}
    \end{minipage}%
\end{figure}

\subsection{Mutual Information (MI) testing experiments}\label{sec:experiment:mitest}

In this subsection, we present our mutual information testing experiment results. We start with the realizable setting, then move to the additive estimation experiment.

\subsubsection{Realizable setting}\label{sec:experiment:mitest:realizable}

Let us assume $U, V \sim N(0,1)$ be two iid random bits. Then, the realizable structure $R_1$ is defined as follows:
\begin{align*}
 X & \gets U \\
 Y & \gets V \\
\end{align*}

We conducted experiments to investigate the impact of varying sample sizes on empirical mutual information. For each fixed value of \( m \), we calculated the average empirical mutual information over 1000 trials. Subsequently, we performed a regression analysis to determine the slope of the linear regression line fitted to the relationship between \( \log(m) \) and \( \log(\text{MI}) \). Our observations indicate that the slope of this line is approximately \( -1.003 \).

\subsubsection{Additive estimation of Mutual Information}\label{sec:experiment:mitest:nonrealizable}

Let us assume $U, V \sim N(0,1)$ be two iid random bits. Then, the non-realizable structure $R_2$ is defined as follows:
\begin{align*}
 X & \gets U \\
 Y & \gets X + V \\
\end{align*}
We conducted experiments to investigate the impact of varying sample sizes on empirical mutual information. For each fixed value of \( m \), we calculated the average empirical mutual information over 100 trials. Subsequently, we performed a regression analysis to determine the slope of the linear regression line fitted to the relationship between \( \log(m) \) and \( \log(\text{MI}) \). Our observations indicate that the slope of this line is approximately \( -0.5214 \). However, it is not completely clear to us why the variance is high in this experiment.

\begin{figure}[!htb]
    \centering
    \begin{minipage}{.5\textwidth}
        \centering
     \includegraphics[width=\linewidth]{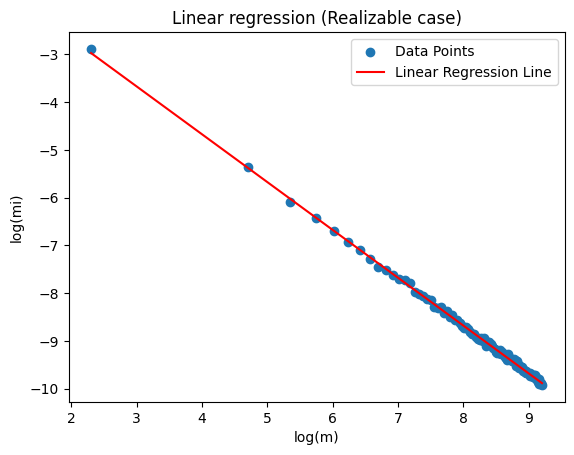}
        \caption{MI testing for realizable case}
        \label{fig:log(m)_vs_log(mi)}
    \end{minipage}%
    \begin{minipage}{0.5\textwidth}
        \centering
    \includegraphics[width=\linewidth]{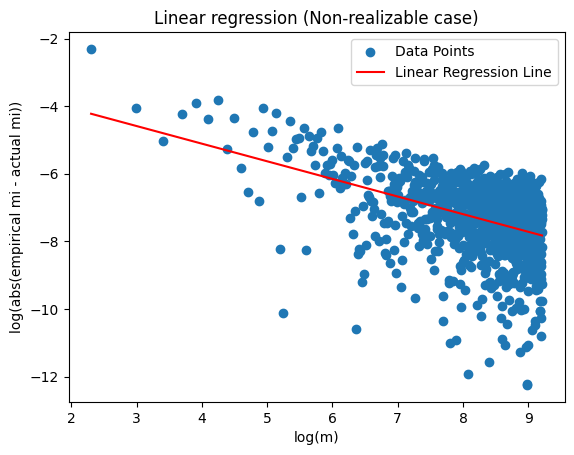}
        \caption{Additive estimation of MI}
        \label{fig:log(m)_vs_log(mi)_est}
    \end{minipage}
\end{figure}

\subsection{Conditional Mutual Information (CMI) testing experiments}\label{sec:experiment_cmi}

Similar to the mutual information testing and estimation in the previous subsection, here we present the conditional mutual information testing results. We start with the realizable setting, and then move to the additive estimation problem.

\subsubsection{Realizable setting}\label{sec:experiment:cmitest:realizable}

Let us assume $Z, U, V \sim N(0,1)$ be three iid random bits. Then, the realizable structure $R_1$ is defined as follows:
\begin{align*} 
        Z & \sim N(0,1) \\
        V & \sim N(0,1) \\
        U & \sim N(0,1) \\  
        X & \longleftarrow  Z + V\\
        Y & \longleftarrow Z + U \\
\end{align*}

We conducted experiments to investigate the impact of varying sample sizes on empirical conditional mutual information. For each fixed value of \( m \), we calculated the average empirical conditional mutual information over 1000 trials. Subsequently, we performed a regression analysis to determine the slope of the linear regression line fitted to the relationship between \( \log(m) \) and \( \log(\text{MI}) \). Our observations indicate that the slope of this line is approximately \( 1.0023 \).

\subsubsection{Additive estimation of Conditional Mutual Information}\label{sec:experiment:cmitest:nonrealizable}

Let us assume $Z,U, V \sim N(0,1)$ be three iid random bits. Then, the non-realizable structure $R_2$ is defined as follows:
\begin{align*} 
        Z & \sim N(0,1) \\
        V & \sim N(0,1) \\
        U & \sim N(0,1) \\  
        X & \longleftarrow  Z + V\\
        Y & \longleftarrow X + Z + U \\
    \end{align*}

We conducted experiments to investigate the impact of varying sample sizes on empirical conditional mutual information. For each fixed value of \( m \), we calculated the average empirical conditional mutual information over 100 trials. Subsequently, we performed a regression analysis to determine the slope of the linear regression line fitted to the relationship between \( \log(m) \) and \( \log(\text{MI}) \). Our observations indicate that the slope of this line is approximately \( -0.5153 \). However, it is not completely clear to us why the variance is high in this experiment.

\begin{figure}[!htb]
    \centering
    \begin{minipage}{.5\textwidth}
        \centering
     \includegraphics[width=\linewidth]{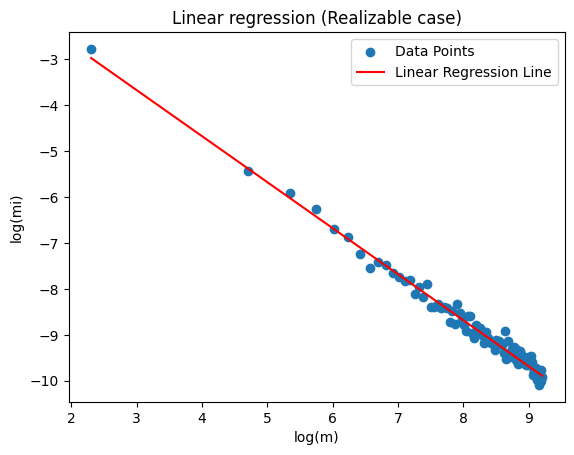}
        \caption{CMI testing for realizable case}
        \label{fig:log(m)_vs_log(cmi)}
    \end{minipage}%
    \begin{minipage}{0.5\textwidth}
        \centering
    \includegraphics[width=\linewidth]{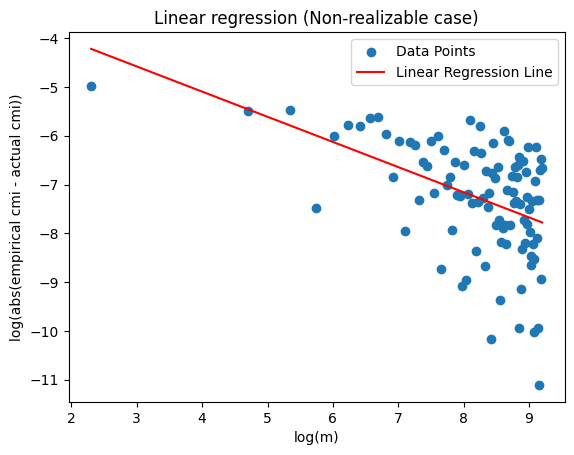}
        \caption{Additive estimation of CMI}
        \label{fig:log(m)_vs_log(cmi)_est}
    \end{minipage}
\end{figure}

\section{Conclusion}\label{sec:conclusion}
In this work, we designed novel (conditional) mutual information testers for Gaussian random variables and used them to design efficient structure learning algorithms for Gaussian tree models. There are several open problems.
An interesting open problem here is the \emph{robust learning} problem, where adversarially corrupted samples from the unknown distribution are obtained.  Another open problem is to study this problem for continuous models other than the Gaussian model.

\section*{Acknowledgements}
We would like to thank the anonymous reviewers of CODS-COMAD'24 for their suggestions which improved the presentation of the paper.
A version of this work excluding the experimental section was earlier submitted to the IJCAI 2024 conference. We are also thankful to the anonymous reviewers of IJCAI 2024 for suggesting us to experiment with our algorithms which led to~\Cref{sec:experiments_full}.

SG's work is partially supported by the SERB CRG Award CRG/2022/00798. SS's research is supported by the National Research Foundation, Singapore and A*STAR under its Quantum Engineering Programme NRF2021-QEP2-02-P05. SS would like to thank the Computer Science department at IIT Kanpur and the IIT Kanpur Initiation grant of SG for the hospitality during his visit, where this work was initiated.

\bibliographystyle{alpha}
\bibliography{reference}

\end{document}